\documentclass[ss]{imsart}

\RequirePackage[OT1]{fontenc}
\RequirePackage{amsthm,amsmath,amssymb,bm}
\RequirePackage[numbers]{natbib}
\RequirePackage[colorlinks,citecolor=blue,urlcolor=blue]{hyperref}

\pubyear{2006}
\volume{0}
\issue{0}
\firstpage{1}
\lastpage{8}

\startlocaldefs
\numberwithin{equation}{section}
\theoremstyle{plain}

\endlocaldefs

\usepackage[letterpaper]{geometry}
\usepackage[parfill]{parskip}
\usepackage{alex}
\usepackage{mathtools}
\usepackage{cases}
\usepackage{amsfonts}       
\usepackage{xr-hyper}
\usepackage[utf8]{inputenc} 
\usepackage[T1]{fontenc}    
\usepackage{nicefrac}       
\usepackage{booktabs}       
\usepackage{comment}
\usepackage{multibib}
\newcites{app}{References}
\usepackage[english]{babel}
\usepackage{mdwlist}
\usepackage{xspace}
\usepackage{multirow}
\usepackage{microtype}
\usepackage{subfigure}
\usepackage{algorithm}
\usepackage[noend]{algpseudocode}
\usepackage{color}
\usepackage{appendix}

\usepackage{url}
\usepackage{hyperref}       
\usepackage{cleveref}
\crefformat{equation}{(#2#1#3)}
\crefrangeformat{equation}{(#3#1#4) to~(#5#2#6)}
\crefname{equation}{}{}
\Crefname{equation}{}{}

\crefname{definition}{\textbf{definition}}{definitions}
\Crefname{definition}{Definition}{Definitions}
\crefname{assumption}{\textbf{assumption}}{assumptions}
\Crefname{assumption}{Assumption}{Assumptions}

\definecolor{maroon}{RGB}{192,80,77}

\newtheorem{theorem}{Theorem}
\newtheorem{lemma}[theorem]{Lemma}
\newtheorem{proposition}[theorem]{Proposition}

\newtheorem{definition}[theorem]{Definition}
\newtheorem{example}[theorem]{Example}

\newtheorem{remark}[theorem]{Remark}

\newcommand{\minimize}{\mathop{\mathrm{minimize}}}
\newcommand{\red}[1]{\textcolor{red}{#1}}

\def\P{\mathbb{P}}

\def\sign{\mathrm{sign}}

\def\diag{\mathrm{diag}}

\def\R{\mathbb{R}}

\def\cF{\mathcal{F}}

\def\cL{\mathcal{L}}

\def\cP{\mathcal{P}}

\def\cS{\mathcal{S}}
\def\cT{\mathcal{T}}

\def\TV{\mathrm{TV}}

\begin{document}

\begin{frontmatter}
\title{Attributing Hacks}
\runtitle{Attributing Hacks}

\begin{aug}
\author{\fnms{Ziqi} \snm{Liu}\thanksref{t1}\ead[label=e1]{ziqiliu@gmail.com}}
\author{\fnms{Alexander} \snm{Smola}\thanksref{t2}\ead[label=e2]{alex@smola.org}}
\author{\fnms{Kyle} \snm{Soska}\thanksref{t3}\ead[label=e3]{ksoska@cmu.edu}}
\\
\author{\fnms{Yu-Xiang} \snm{Wang}\thanksref{t2,t3}\ead[label=e4]{yuxiangw@cs.cmu.edu}}
\author{\fnms{Qinghua} \snm{Zheng}\thanksref{t4}\ead[label=e5]{qhzheng@xjtu.edu.cn}}
\and
\author{\fnms{Jun} \snm{Zhou}\thanksref{t1}\ead[label=e6]{jun.zhoujun@alibaba-inc.com}}


\thankstext{t1}{AI Department, Ant Financial Group, \{ziqiliu,jun.zhoujun\}@antfin.com}
\thankstext{t2}{Amazon AWS, \{smola,yuxiangw\}@amazon.com}
\thankstext{t3}{Carnegie Mellon University, ksoska@cmu.edu, yuxiangw@cs.cmu.edu}
\thankstext{t4}{Xi'an Jiaotong University, qhzheng@xjtu.edu.cn}
\runauthor{Ziqi Liu et al.}

\affiliation{Some University and Another University}

\end{aug}

\begin{abstract}
In this paper we describe an algorithm for estimating the provenance
of hacks on websites. That is, given properties of sites and the
temporal occurrence of attacks, we are able to attribute individual
attacks to joint causes and vulnerabilities, as well as estimating the evolution of these vulnerabilities over time.
Specifically, we use hazard regression with a time-varying additive hazard function parameterized in a generalized linear form. The activation coefficients on each feature are continuous-time functions over time. We formulate the problem of learning these functions as a constrained variational maximum likelihood estimation problem with total variation penalty and show that the optimal solution is a $0$th order spline (a piecewise constant function) with a finite number of adaptively chosen knots. This allows the inference problem to be solved efficiently and at scale by solving a finite dimensional optimization problem.
Extensive experiments on real data sets show that our method significantly outperforms Cox's proportional hazard model.  We also conduct case studies and verify that the fitted functions
of the features respond to real-life campaigns.
\end{abstract}


\tableofcontents
\end{frontmatter}

\section{Introduction}

Websites get hacked whenever they are subject to a vulnerability that
is known to the attacker, whenever they can be discovered efficiently,
and, whenever the attacker has efficient means of hacking at his
disposal. This combination of \emph{knowledge}, \emph{opportunity},
and \emph{tools} is quite crucial in shaping the way a group of sites
receives unwanted attention by hackers.

Unfortunately, as an observer we are not privy any of these
three properties. In fact, we usually do not even know the exact time
$t_s$ a site $s$ was hacked. Instead, all we observe is that a
compromised site will eventually be listed as such on one (or more)
blacklists. That is, we know that by the time a site lands on the
blacklist it definitely has been hacked. However, there is no
guarantee that the blacklists are comprehensive nor is there any
assurance that the blacklisting occurs expediently. Another shortcoming of blacklists
is that they do not reveal which aspect of the website was to blame.

On the other hand, metadata exists for each website and it allows us
to measure the potential vulnerability of the websites quantitatively. This
includes specific string snippets on websites that are indicative of
certain versions of software which might have been identified as
exploitable or containing bugs that lead to possible security breach. An
interesting method that uses these features to identify websites at
risk was recently proposed by \citet{Soska2014}. However, it is unclear
how each of these features contribute to the ``hazard'' of a
particular website getting hacked at a given time. 

In this paper, we propose a novel hazard regression model to
address this problem. Specifically, the model provides a clear
description of the probability a site getting hacked conditioned
on its time-varying features, therefore allowing prediction tasks
such as finding websites at risk, or inferential tasks such as
attributing attacks to certain features as well as identifying
change points of the activations of certain features to be
conducted with statistical rigor.

\paragraph{Related work.} The primary strategy for identifying web-based malware has been to detect an active infection based on features such as small iFrames~\citep{Provos:Security08}. This approach has been pursued by both academia \citep[e.g.,][]{Borgolte:CCS13,Invernizzi:S&P12} and industry \citep[e.g.,][]{safebrowsing,siteadvisor,NortonSafeWeb}. While intuitive, this approach suffers from being overly reactive, and defenders must compete against adversaries in an arms race to detect increasingly convoluted and obscure forms of malice.

\citet{Soska2014} propose a data driven (linear classification) approach to identify
software packages that were being targeted by attackers to predict the
security outcome of websites.

Compared with \cite{Soska2014}, our method
is able to predict the time a site will be hacked in a survival analysis framework.
Our method naturally handles censoring of observations
(i.e.\ inconsistency of exact hacking time and the time listed on
blacklists), automatically identifies a small number of features as exploits and allows the activation
coefficients on each feature to be functions over continuous time. 


Finally, our hazard regression model is quite generic and much more powerful than the widely-used Cox model \citep{Cox1972} in our application, therefore it can be viewed as a novel and alternative way of estimating nonparametric
hazard functions at scale, and be used as a drop-in replacement in many other applications having similar structures. Towards the end of the paper, we provide one such application on studying the user dropout rate (churn rate) of Alipay --- a major online payment services used by hundreds of millions of people --- and illustrate some interesting insight.


\section{Background}

Our work is based on two key sets of insights: the specific way how
vulnerabilities are discovered, exploited and communicated in the
community, and secondly, the mapping of these findings into a specific
statistical model. 

\subsection{Attacks on Websites}
We start by describing the typical economics of hackers and websites.

Exploits are first discovered by highly skilled individuals
(hackers) who tend to reserve these exploits for their own purposes for an extended
period of time. Typically, a hacker will retain an exploit as long as there is an ample supply of vulnerable sites
that can be discovered efficiently. Once the \emph{opportunity} for
such hacks diminishes due an exhausted supply, the exploits are often
sold or freely published to the benefit of the author's reputation.

Once this knowledge enters the public domain, the availability of
available tools increases with it. It is added to the repertoire of
popular kits, at the ready disposal of ``script kiddies'' who will
attempt to attack the remaining sites. The increased availability of
\emph{tools} often offsets the reduced \emph{opportunity} to yield a
secondary wave of infections. 

An important aspect in the above scenario is the way how sites are
discovered. Quite frequently this is accomplished by web queries for
specific strings in sites, indicative of a given vulnerability (e.g.\
database, content management system (CMS), server, scripting language). In other words, string
matches are excellent \emph{features} to determine the vulnerability
of a site and are therefore quite indicative of the likelihood that
such a site will be attacked. Unfortunately, we are not privy to the
search strings a potential attacker might issue. However, we can use
existing fingerprints to \emph{learn} such sequences, e.g.\ the tags
and attributes in the pages of a site.

In a nutshell, the above informs the following statistical
assumptions on the nature of website vulnerabilities.
Firstly, sites are only effectively vulnerable once an exploit has been discovered. 
Second, changes in attack behavior are discrete
rather than gradual. In the following we design a statistical
estimator capable of adapting to this specific profile.


\subsection{Hazard Regression}

Hazard regression is commonly used in survival analysis of patients
suffering from potentially fatal diseases. There, one aims to estimate
the chances of survival of a particular patient with covariates
(attributes) $x$, as a function of time, such as to better understand
the effects of $x$. Unfortunately, each patient only has one life, and
possibly different attributes $x$, hence, it is impossible to estimate
the fatality rate directly. 

Instead, one assumes that the hazard rate $\lambda(x,t)$ governs the
instantaneous rate of dying of any $x$ at any given time $t$:
\begin{align}
  \lambda(t) = \lim_{dt\to 0} \frac{p(t\leq T < t+dt|T\geq t)}{dt} 
  = \lim_{dt\to 0} \frac{p(t\leq T < t+dt)}{dt} \cdot \frac{1}{p(T\geq t)}
\end{align}
That is, the density of dying
at time $t$ is given by 
\begin{align}
\label{eq:hazarddensity}
p(t|x) = \lambda(x,t)\underbrace{p\rbr{\text{survival until $t$}|x}}_{F(t|x)}.
\end{align}
This leads to a differential equation for the survival probability
with solution
\begin{align}
\label{eq:hazardintegral}
F(t|x) = \exp\rbr{-\int_0^t \lambda(x,\tau) d\tau}.
\end{align}
Here we assumed, without loss of generality, that time starts at
$0$. Note that a special case of the above is $\lambda(x,t) =
\lambda_0$, in which case we have $F(t,x) = e^{-t \lambda_0}$. This is
the well-known nuclear decay equation (also an example of survival
analysis). 

In our case, death amounts to a site being infected and $\lambda(x,t)$
is the rate at which such an infection occurs. An extremely useful
fact of hazard regression is that it is additive. That is, if there
are two causes with rates $\lambda$ and $\gamma$ respectively,
\eq{eq:hazarddensity} allows us to add the rates. We tacitly assume
here that once a site is infected, the attacker will take great care
to keep further attackers out, or at least, it will remain blacklisted
as long as it is infected in some manner.
In terms of \eq{eq:hazardintegral} we have 
\begin{align}
\label{eq:additivehazard}
F(t|x) = \exp\rbr{-\int_0^t \lambda(x,\tau) + \gamma(x,\tau) d\tau}\\\nonumber
\text{and\,\,\,\,\,\,\,\,\,\,}
p(t|x) = \sbr{\lambda(x,t) + \gamma(x,t)} F(t,x).
\end{align}
The reason why this is desirable in our case follows from the fact
that we may now model $\lambda$ as the sum of attacks and can treat
them as if they were independent in the way they affect sites. 

One challenge in our analysis is the fact that we may not always
immediately discover whether a site has been taken over. The
probability that this happens in some time interval $[t_1, t_2]$ is
given by $F(t_1|x) - F(t_2|x)$, i.e.\ by the difference between the
cumulative distribution functions. 

Finally, the absence of evidence (of an infection) should not be
mistaken as evidence of absence of such. In other words, all we know
is that the site survived until time $T$. By construction, their
probability is thus given by $F(T|x)$. In summary, given intervals
$[t_i, T_i]$ of likely infection for site $i$, at time $T$ we have the
following likelihood for the observed data:

\begin{align}
p(\mathrm{sites}|T) = \prod_{i \in \mathrm{hacked}} \sbr{F(t_i,
	x_i) - F(T_i, x_i)} \prod_{i \not\in \mathrm{hacked}} F(T,x_i).
\end{align}

Up to this point, the model is completely general and little assumptions are made. The key of statistical modeling boils down to specify a tractable parametric or nonparametric form of the hazard rate function $\lambda(x,t)$, which by construction only needs to be nonnegative and obeys that $\int_{0}^\infty \lambda(x,t)dt =\infty$ for all possible $x$.

Arguably the most commonly used hazard regression model is the Cox's proportional hazard model, where $\lambda(t|x)=\lambda_0(t) \exp(w^\top x)$~\cite{Cox1972}. The linear dependence on feature vector $x$ makes it very appealing for interpretability and inference and leaving $\lambda_0(t)$ unspecified allows the model to handle global variations over time that is not captured by covariates $x$.

There is a large body of work devoted to extending the Cox model by coming with 
useful specifications of the hazard rate as more generic functions of
the covariate $x$ and time $t$ \cite{sauerbrei2007new,Bradic2012,perperoglou2013reduced}, and inventing techniques to address the curse-of-dimensionality associated with nonparametric modeling \cite{Tibshirani1997,verweij1995time,perperoglou2014cox,hastie1990generalized}. 
We cannot possibly enumerate these work exhaustively, so we encourage curious readers to check out the wonderful textbooks \cite{klein2005survival,therneau2013modeling}, the documentation for the ``survival'' package in R \cite{therneau2017package} and the references therein. Part of our contribution in this paper is to connect various bits and pieces of the statistical literature for the task of modeling hacker activities.

\subsection{Trend Filtering}
Trend filtering \citep{Kim2009,Tibshirani2014} is a class of nonparametric regression estimators that has precisely the required property. It is minimax optimal for the class of functions $[0,1]\rightarrow \R$ whose $\alpha$th order derivative has bounded total variation. In particular, it has the distinctive feature that when $\alpha=0,1$ it produces piecewise constant and piecewise linear estimates (splines of order $0$ and $1$) and when $\alpha\geq 2$ it gives piecewise smooth estimates. The local adaptivity stems from the sparsity inducing regularizers that chooses a small but unspecified number of knots. When $\alpha=0$, trend filtering reduces to the fused lasso \citep{tibshirani2005sparsity} which solves
$$
\argmin_\beta \cL(\beta)  + \gamma \sum_{\ell=1}^{T-1} |\beta_{\ell+1}-\beta_\ell|.
$$
for a given loss function $\cL$.
The advantage of this model is that each discrete change in the rate
function effectively corresponds to the discovery or the increased (or decreased)
exploitation of a vulnerability --- after all, the \emph{rate} of infection
should not change unless new vulnerabilities are discovered or a patch is released.

There has been an increasingly popular body of work on extending of trend filtering to estimate functions with heterogeneous smoothness over graphs \citep{wang2016trend}, in multiple dimensions (e.g., images, videos) \citep{sadhanala2016total} and with additive structures \citep{sadhanala2017additive}. It will be clear later that our proposed method can be considered an additive survival trend filtering model, which as far as we know, is the first time trend filtering (or additive trend filtering) is used for (time-varying) survival analysis.


\section{Attributing Hacks}

We will now assemble the aforementioned tools into a
joint model for attributing hacks.

\subsection{Additive hazard function and variational maximum likelihood}\label{sec:stat}
Given the hazard function $\lambda(t,x_i)$
of each website $i \in \{1,...,n\}$ with feature vector ${x}_i(t) \in \RR^d$ at time $t$, we have the
following survival problem:
\begin{equation*}
{\max} \prod_{i \in \mathcal{B}} p(t_i \leq \tau_i < T_i)
\prod_{i \notin \mathcal{B}} p(\tau_i > T)
\end{equation*}
where $\tau_i$ is the \emph{unobserved} random variable indicating the exact
time that website $x_i$ is being hacked. All we know for websites on
the blacklist\footnote{A blacklist in this context is a list of websites maintained by a third party which are confirmed to be either malicious, compromised, or otherwise adversarial according to the expertise of the curator. Entries in a blacklist always contain the website in question, but are also furnished timestamps of the security event and information regarding the nature of the malice.} is the time already
been hacked $T_i$ and the last time it was alive $t_i$. This
is what we call an ``interval-censored'' observation.
Time $T$ denotes the end of the observation interval, e.g., now. Websites that were still alive at $T$ are considered ``right censored'' because all we know is that their hack time will be beyond of $T$. Under the survival analysis framework, we have
\begin{align}
  \nonumber
  p(\tau_i > T) =& e^{-\int_0^{T} \lambda(t,x_i(t))dt} \\
  p(t_i \leq \tau_i < T_i) =& p(\tau_i \geq t_i) - p(\tau_i \geq T_i) 
  = e^{-\int_0^{t_i} \lambda(t,x_i(t)) dt} 
  - e^{-\int_{0}^{T_i}\lambda(t, x_i(t)) dt}\label{eq:cumulative_repre}
\end{align}
It remains to specify the hazard function. In our setting, $x$ is a high-dimensional non-negative feature vector, so we need to impose further structures on the hazard function $\lambda(x,t)$ to make it tractable. We thus make an additive assumption and expand the hazard function into an inner product 
$$
\lambda(x,t) = \langle x(t), w(t)\rangle  = w_0(t)+\sum_{i=1}^d x_i(t)w_i(t). 
$$
This is still an extremely rich class of functions as $x_i(t)$ can be different over time and $w_i(t)$ is allowed to be any univariate nonnegative functions\footnote{A natural alternative formulation would be to take $\lambda(x,t) = \exp(w_0(t)+\sum_{i=1}^d x_i(t)w_i(t))$ as in Cox's model, so we do not have to impose any constraints on feature $x$ and coefficient vector $w$. In this paper, since the features are all $\{0,1\}$ indicators of the existence of certain html snippets, so we naturally have nonnegative feature vectors and also it seems more reasonable to assume the contribution to the hazard from each feature adds up linearly, rather multiplied together exponentially.}. Leaving it completely unconstrained will inevitably overfit any finite data set. However, standard nonparametric assumptions on $w_i(t)$ would require the function to be Lipschitz continuous or even be higher order differentiable. These assumptions makes it a poor fit for modeling the sharp changes of $w_i(t)$ in response to sparse events such as an release of an exploit on hacker forums.

\paragraph{Standard and monotone model with Total Variation penalty.}
To address this issue, we propose to penalize the total variation (TV) of $w_i$ for each $i$ and optionally, impose a monotonicity constraint on these functions.
This gives rise to a variational penalized maximum likelihood problem below:
\begin{equation}
\begin{aligned}\label{eq:variational}
&\minimize_{(w_0,w_1,...,w_d) \in \cF^{d}} \;&&\;\sum_{i=1}^n\ell(\{x_i, z_i, \psi_i\}; {w}) + \gamma\sum_{j=0}^d \TV(w_j) \\
&\text{Subject to:} \;&&\; w_j(t)\geq 0  \text{ for any } j\in [d],t\in \R.\\
& \;&&\text{ (if monotone) }\;w_j(t+\delta) -w_j(t)\geq 0
\text{ for any } j\in [d],t\in \R, \delta\in \R_+ 
\end{aligned}
\end{equation}
where $\ell$ is the negative log-likelihood functional; $z_i$ is the indicator of censoring type for observation $x_i$, i.e.
interval-censored or right-censored; $\psi_i := \{t_i, T_i, T\}$ 
is the associated censoring time;$\cF$ are the set of  all functions $[0,T]\rightarrow \R$; and $w_j(t)$ is the evaluation of
function $w_j$ at time $t$. We define total variation as follows:
\begin{definition}[Total Variation]
	The total variation of a real-valued function $f$ defined on interval $[a,b]\subset \R$ is 
	$$
	\TV(f) = \sup_{\cP\in \left\{  P=\{t_0,...,t_{n_P}\} \middle| P \text{ is a partition of }[a,b]\right\}} \sum_{i=1}^{n_P-1} |f(t_{i+1})-f(t_i)|.
	$$
\end{definition}
When $f$ is differentiable almost everywhere, total variation reduces
to the $L_1$ norm of the first derivative
$$
\nbr{f}_{L_1} = \int_a^b \abr{\partial_t f(t)} dt.
$$
This should be compared to the squared $L_2$ norm penalty with degree $1$, commonly used as a variational form of the smoothing splines:
$$
\nbr{f}_{L_2}^2 =  \int_{a}^b  \sbr{\partial_t f(t)}^2 dt.
$$
By construction, a piecewise constant function $f$ would have small
total variation but unbounded $L_2$ norm of its first derivative. 

We now make two remarks on the constraints.
First of all, the nonnegativity constraints is necessary to ensure the hazard function $\lambda(x)$ to be a valid hazard function for all $x\geq 0$. Secondly, whether the monotonic constraints are needed is completely application specific. We call the model with or without the monotonic constraints the ``monotone model'' and ``standard model'' respectively. There could also be a ``mixed'' model where we impose monotonic constraints only on a subset of the features where there are reasons to believe the corresponding hazard only increase over time.
From the modeling point of view, when we impose the ``monotone'' constraints, it trivializes the total variation penalty as $\TV(w_j) =  w_j(T)-w_j(0)$; as a result, in cases where monotonicity of hazard makes sense, we often take $\gamma = 0$ which makes the model attractively parameter-free. 

\paragraph{Non-convex variant of total variation.}
While the total variation penalty is convex and it often induces a
sparse number of change points, it also results in substantial
estimation bias due to the shrinkage of the magnitude of the
changes. The issue is exacerbated when we use a monotonicity
constraint. Since $\TV(w_j) = w_j(T)-w_j(0)$, the total variation
remains the same whether there is only one change point or there are
infinite number of change points. This is a nice feature as the class
of functions contains functions that are smoothly changing, but on the
other hand, it does not always give us a sparse number of change
points that are useful for interpreting the results. This motivates us
to consider a non-convex variant of the total variation penalty of the
following form:
\begin{equation}\label{eq:log-penalty}
\TV_{\log}(f) :=\sup_{\cP\in \left\{  P=\{t_0,...,t_{n_P}\} \middle| P \text{ is a partition of }[a,b]\right\}} \sum_{i=1}^{n_P-1} \log(\epsilon+|f(t_{i+1})-f(t_i)|).
\end{equation}
where $\epsilon$ is a tuning parameter. Taking $\epsilon=1$ will ensure that the functional being non-negative and is $0$ only when $f$ is a constant function. 
We call this variant of the model the ``log'' model or the ``log +
Monotone'' when the monotonic constraints are imposed. The new
functional has a number of desirable properties:
\begin{lemma}
  For any function $f$ we have that $\TV_{\log}(f) \leq
  \TV(f)$. Moreover, if $f$ is Lipschitz continuous it follows that
  $\TV_{\log}(f) = \TV(f)$. 
\end{lemma}
\begin{proof}
  We use two elementary properties of $\log(1+x)$, which follow from 
  concavity. Firstly, the tangent at $0$ majorizes the function, i.e.\
  $\log (1+x) \leq x$ for all $x >
  -1$. Secondly, for any $x,y \geq 0$ we have that $\log(1 + x + y)
  \leq \log(1 + x) + \log(1+y)$. 

  The first property shows that for any partition of $[a,b]$ the value
  of the supremum in \eq{eq:log-penalty} is smaller than that of its
  counterpart for $\TV$. Hence the supremum is majorized, too. 

  The second property follows from the fact that for $\TV$ we can take
  the limit of infinitesimally small segments without decreasing its
  value. In other words, the limit of an infinite partition has the
  same value as $\TV(f)$. On the other hand, 
  \begin{align}
    \label{eq:bigohlog}
    x + y \leq \log(1 + x + y) + C(x^2 + y^2) \text{ for some } C > 0
  \end{align}
  Hence, any $\epsilon$-partition of the interval $[a,b]$ for a
  Lipschitz continuous function $f$ with Lipschitz constant $L$ will
  yield a value of the $\log(1+x)$ penalty that is at most $L^2
  \epsilon |b-a|$ smaller than $TV(f)$. Hence, for $\epsilon \to 0$
  this converges to $TV(f)$.
\end{proof}
A consequence of this is that $\TV_{\log}$ favors functions that are not
Lipschitz, i.e.\ functions that have a small number of larger jumps,
which is exactly what we want. The following example illustrates this.
\begin{example}
	 Consider thre functions $f,g,h: [0,1]\rightarrow \R$. Let $f= I(t\leq 0.5)$, $g= 0.5I(t\leq 1/3) + 0.5 I(t\leq 2/3)$, and $h= t$.
	 Note that $f$ has one change point, $g$ has two and $h$ has an infinite number of change points.
	 
	 Assuming $\epsilon\leq 1$, then 
	 $$\TV_{\log}(f) = \log(\epsilon + 1)$$
	 and the optimal partition is a single block at $t_0=0$, $t_1=1$. It is clear that since $\epsilon\leq 1$ the more block we partition it into, it can only make it smaller. Turns out that we can also calculate the expression for $g$ and by Jensen's inequality and concavity of $\log$, we can clearly see that
	 $$
	 \TV_{\log}(g) = 2\log(\epsilon + 0.5) \geq \log(\epsilon + 1).
	 $$
	 Now
	 $$
	 \TV_{\log}(h) \geq \sup_{\delta \in (0,1]} \frac{1}{\delta}\log(\epsilon + \delta).
	 $$
	 When $\epsilon=1$, it is not hard to check that the supremum is achieved at $\delta\rightarrow 0$, since 
	 $$\left[\frac{1}{\delta}\log(1 + \delta)\right]' = \frac{1}{\delta^2}\left(\frac{\delta}{1+\delta} - \log(1+\delta)\right)\leq 0.$$
	 	 For other $\epsilon\leq 1$, there exists and optimal $\delta>0$ that maximizes the expression.
	 However, taking $\delta\rightarrow  0$ will provide a valid lower bound for all $0<\epsilon\leq 1$.
	 	 By l'Hospital's rule, we get
	 	 $$
	 	 \TV_{\log}(h) \geq \frac{1}{\epsilon},
	 	 $$
	 	 which is bigger than $\TV_{\log}(g)$ with two change points (and $\TV_{\log}(f)$ with one-change points). We conjecture that it is also bigger than any the $\TV_{\log}$ of any functions with a finite number of change points but the same total variation.
\end{example}
We will experiment with different variants of the model in simulation and real data experiments to illustrate the merits of the various approaches.

\paragraph{Comparison to the Cox model.}
Note that our model (regardless of which variant) is a much richer representation comparing to
Cox's proportional hazard model \citep{Cox1972}. Our method
handles time-varying coefficients and feature vectors while Cox
model is static. Also, the semiparametric nature of Cox model by
construction leaves out the baseline hazard $\lambda_0(t)$ such
that it becomes non-trivial to produce a proper survival
distribution. For example, a common trick is to parametrize the
baseline hazard rate $\lambda_0(t)$ by a either a constant or a log-Weibull
density. Our formulation does not require a parametric
assumption and produces a nonparametric estimate of
it to account for all the effects that are not explained by the
given feature. 

The only remaining issue is that \eqref{eq:variational} is an infinite dimensional functional optimization problem and could be very hard to solve. In the subsequent two sections, we address this problem by reducing the problem to finite dimension and deriving scalable optimization algorithms to solve it.


\subsection{Variational characterization}

The following theorem provides a finite set of simple basis functions
that can always represent at least one of the solutions to
\eqref{eq:variational}.

\begin{theorem}[Representer Theorem]\label{thm:repre}
  Assume all observations are either right censored or interval censored \footnote{This is not restrictive because any uncensored data point can be made interval censored with a tiny interval.}, feature $x_i(t)$ for
  each site is nonnegative, piecewise constant over time with finite number
  of change points. Let $s_{\tau}(t) = 1(t\geq \tau)$ be the
  step function at $\tau$. Then there exists an optimal solution
  $(w_1^*,...,w_d^*)$ of Problem \eqref{eq:variational} (either standard model or monotone model) such that for each
  $j=1,..,d$, $$w_j^*(t) = \sum_{\tau\in \cT} s_{\tau}(t)c_{\tau}^{(j)}$$ 
	for some set $\cT$ that collects all censoring boundaries and time steps where feature $x_j(t)$ changes. The coefficient vector $c^{(j)}\in \R^{|\cT|}$.
\end{theorem}
The proof, given in the appendix, uses a variational characterization due to \citet{deboor1978practical,mammen1997locally} and a trick that reparameterizes our problem using the cumulative function $W_i(t) = \int_0^{t} w_i(t)dt$. Extra care was taken to handle the non-negativity and monotone constraints. We remark that the above result also applies trivially to the case when $\gamma = 0$ (unpenalized version). 

The direct consequence of Theorem~\ref{thm:repre} is that we can now represent piecewise constant functions by vectors in $\R^{|\cT|}$ and solve \eqref{eq:variational} by solving a tractable finite dimensional fused lasso problem (with a nonnegativity constraint and possibly a monotonic constraint) of the form:
\begin{equation}\label{eq:finite_optim}
\begin{aligned}
&\minimize_{w_0,w_1,...,w_d \in \R_+^{|\cT|}} \;&&\;\sum_{i=1}^n\ell(\{x_i, z_i, \psi_i\}; {w}) + \gamma\sum_{j=0}^d \|D w_j\|_1 \\
&\text{subject to:} \;&& w_j(\ell)\geq 0  \text{ for any } j\in [d],\ell = 1,...,|\cT|.\\
&\;&& \text{(if monotone) } w_j(\ell+1) -w_j(\ell)\geq 0 \text{ for any } j=1,...,d, \ell = 1,...,|\cT|-1. 
\end{aligned}
\end{equation}
where we abuse the notation to denote $w_j$ as the vector of evaluations of function $w_j$ at the sorted time points in $\cT$; and $D\in \R^{(|\cT|-1)\times |\cT|}$ is the discrete forward difference operator. 

Although the above result does not cover the cases when we replace the TV-penalty with the nonconvex the log penalty described in \eqref{eq:log-penalty}.
we can still restrict our attention
to the class of piecewise constant functions, and penalize 
\begin{equation}\label{eq:discrete-log}
\tilde{\TV}_{\log}^\epsilon(w_j) := \sum_{\ell=1}^{|\cT|-1} \log \big(\epsilon+|w_j(\ell+1) - w_j(\ell)| \big)
\end{equation}
instead of \eqref{eq:log-penalty}. 
We claim that this is a sensible approximation because this is effectively choosing a specific partition according to the time points that comes from the data, therefore is always a lower bound of \eqref{eq:log-penalty}. In fact, when $\epsilon=1$ and we can check that the number of non-zero summand in \eqref{eq:discrete-log} is exactly the same as the number of jumps in the fitted function $w_j$ and for the same total variation of $w_j$, this penalty would be smaller when the number of change points are larger.

\begin{remark}[Higher order Trend Filtering]
	For $k$-th order trend filtering with $k \geq 1$, we do not get the same variational characterization. Although it is still true that the optimal solution set contains a spline $W_j^*$, there is no guarantee that the knots of $W_j^*$ are necessarily a subset of $\cT$.
	Fortunately, by Proposition~7 of \citet{mammen1997locally}, restricting our attention to the class of splines with knots in $\cT$ will yield a solution that is very close to the $W_j^*$ at every $\tau\in\cT$ and it has total variation of its $k$th derivative on the same order as $\TV(W_j^*)$. In addition, a spline is uniformly approximated by the class of functions that can be represented by the falling factorial basis \citep{Tibshirani2014,wang2014falling}, therefore, the function from $k$th order trend filtering defined on $\cT$ will be a close approximation to the optimal solution of the original variational problem.
\end{remark}
\begin{remark}[A sparse and memory-efficient update scheme]
	The theorem suggests a memory-efficient scheme for optimization, as one can only keep track of the coefficients of the step functions rather than representing the dense vectors $w_0,...,w_d$. Moreover, each stochastic gradient update will be sparse since each user has only a handful of changes in his feature vectors over time and at most two censoring brackets.
\end{remark}
\begin{remark}[Sparsity in the features as implicit regularization]
	On top of that, while the feature vectors are ultra high-dimensional, they are extremely sparse. When $x_j(t) = 0$, the additive structure of the model ensures that changes in $w_j(t)$ do not affect the hazard rate, therefore $w_j(t)$ will be chosen such that the TV penalty is minimized. This suggests that sparsity structures in features also serve as an implicit ``regularization'' which helps to improve the generalization of the model.
\end{remark}

\subsection{Algorithms}
The previous section reduces the optimization over functions to an optimization problem over vectors in finite dimensional Euclidean space. However, it does not imply that the problem is easy to solve. 
Due to the interval-censoring, the loss functions are not convex, and the penalty is either convex ($\TV$) or nonconvex ($\log (|\cdot|+\epsilon)$) but non-smooth. In addition, there are non-negativity and possibly monotonic constraints. In this section, we derive numerical optimization techniques for solving problem \eqref{eq:finite_optim}  and discuss the convergence rate to a stationary point.

Specifically, we propose to train all four variants of our models (standard and monotone / TV or TV-log) using the proximal gradient algorithm with a stochastic variance-reduced gradient approximation \citep{Johnson2013}.
The proposed method is computationally efficient and system-friendly by design, which allows us to scale up the method to work with
hundreds of thousands of data points. While the problem is nonconvex, we find the proposed techniques extremely effective in simulation and real data experiments.

In the subsequent discussions, we will first derive the gradient of the loss function and write down the variance-reduced unbiased estimate of the gradient as per \citep{Johnson2013}. Then we will describe the key steps of the algorithms. 
While the proximal gradient algorithm itself is standard, the main challenge for this specific problem is to be able to solve the proximal map efficiently, since there are several non-smooth terms. We show using a general technique due to \citet{yu2013decomposing} that we can decompose the proximal map into two simpler ones which both admit linear-time algorithms.

\paragraph{Gradient calculations.} Note that the probability of each interval-censored data
$-\log(p(t_i \leq \tau_i < T_i))$ in \eqref{eq:cumulative_repre} can be decomposed
as $\int_0^{t_i} \lambda(t,x_i) dt -\log\Big(1-\exp\big(-\int_{t_i}^{T_i}\lambda(t,x_i)dt\big)\Big)$.
As a result we only need to calculate the integral between $t_i$ and $T_i$ for interval-censored data while evaluating the gradients:
\begin{align*}
\nabla {g_i(w)} =
\begin{cases} 
\frac{1}{1-\exp{\big(\int_{t_i}^{T_i} \lambda(t,x_i)dt\big)}}D\tau^\top&
\text{if }t_i \leq t<T_i.\\
D\tau^\top,& 
\text{otherwise.}\\
\end{cases} \nonumber
\end{align*}
where $g_i(\cdot)$ is the negative log-likelihood function on $x_i$; $\tau \in \RR^{d,\cT}$ denotes sorted times in $\cT$ for each feature $j \in \{1,...,d\}$. 

\paragraph{Proximal-SVRG algorithm.}
The proximal-SVRG algorithm iteratively solves the following with learning rate $\eta$, minibatch size $m$, for $t=0,1,2,3,...$ until convergence:
\begin{enumerate}
	\item If $\text{mod}(t,\lfloor n/m\rfloor) =0$, then for each $j\in[d]$, assign $\tilde{w}_j = w_j$, evaluate the full gradient $\tilde{\mu}_j = \sum_{i=1}^n\nabla g_i(\tilde{w}_j)$.
	Note that the full gradient evaluation is only called every data pass. 
	\item for $j=1,...,d$:
	\begin{enumerate}
		\item Pick a random minibatch $\cS\subset [n]$: 
		$$w^{\textbf{tmp}}_j = w^{(t)}_j - \eta \left(\sum_{i\in \cS}\nabla g_i(w^{(t)}_j) -
		\sum_{i\in\cS}\nabla g_i(\tilde{w}_j) + \tilde{\mu}_j\right).$$
		
		\item Solve the proximal map:
		\begin{small}
		\begin{align}\label{eq:fullprox}
		w^{(t+1)}_j =  \begin{cases}
		\argmin_{w\in\R^{|\cT|}}  \frac{1}{2}\|w - w^{\textbf{tmp}}_j\|^2 + \gamma \|Dw\|_1 +  \delta(w\geq 0), & \text{ for \textbf{standard model}}.\\
		\argmin_{w\in\R^{|\cT|}}  \frac{1}{2}\|w - w^{\textbf{tmp}}_j\|^2 + \gamma \|Dw\|_1 +  \delta(w\geq 0) + \delta(Dw\geq 0),& \text{ for \textbf{monotone model}}.
		\end{cases}
		\end{align}
	\end{small}
		Here $\delta: \R^{|\cT|} \rightarrow \R\cup \{+\infty\}$ is the standard indicator function in convex analysis that evaluates to $0$ when condition is true and $+\infty$ otherwise.
	\end{enumerate}
\end{enumerate}

\begin{proposition}[Decomposable proximal map.]\label{prop:prox_decomp}
	\eqref{eq:fullprox} is equivalent to first solving 
	$$
	w^{\textbf{tmp2}} = \argmin_{w} \|w^{\textbf{tmp}} - w\|_2^2 + \delta(w\geq 0)
	$$
	and then solving
	\begin{equation}\label{eq:prox2-nonmonotone}
	w^{(t+1)}_j =\argmin_w \|w^{\textbf{tmp2}} - w\|_2^2 + \gamma\|Dw\|_1
	\end{equation}
	or
	\begin{equation}\label{eq:prox2-monotone}
		w^{(t+1)}_j =\argmin_{w} \|w^{\textbf{tmp2}} - w\|_2^2 + \gamma\|Dw\|_1 + \delta(Dw \geq 0).
	\end{equation}
	for the standard model and monotone model respectively. 
\end{proposition}
	The proof makes use of Theorem 1 of \citet{yu2013decomposing}, and basic definition and properties of subdifferential in convex analysis (see, e.g., \citep{rockafellar2015convex}). The full proof is presented in the appendix.

%

The above result ensures that we can solve the proximal map in two steps.
The first step update is simply a projection to the first orthant, which involves trivially projecting each coordinate separately to $\R_+$. The second step can also be solved in linear time by a dynamic programming algorithm~\cite{johnson2013dynamic}.

\paragraph{Using the same algorithm for the non-convex penalty.} We now show how to use the same algorithm for the log-model. The idea is that we decompose the non-convex penalty as follows
$$
{\tilde{\TV}}_{\log}^\epsilon(w)  = \frac{\|Dw\|_1}{\epsilon} + \xi(w).
$$
The following lemma describes the convenient property of $\xi(w)$ that allows us to group it into the smooth loss function.
\begin{lemma}\label{lem:tv-log2tv}
$\xi(w)$ is continuously differentiable. Also, the gradient of $\xi$ is 
	$$
	D^T \diag\left(\frac{1}{\epsilon + |Dw|}-\frac{1}{\epsilon}\right)\sign(Dw).
	$$
\end{lemma}
Note that $\xi(w)$ always exists so the only thing we need to prove is that $\xi(w)$ is continuously differentiable.
The proof uses the definition of multivariate differentiability by checking that at all point $w$, all directional derivatives exists. It is mostly elementary calculus but is somewhat long ans technical, so we defer the detailed arguments to the appendix.

The lemma implies that we can rewrite the objective function of the log-model into:
\begin{equation}\label{eq:objective}
\sum_{i=1}^d\ell(\{x_i,z_i,\psi_i\}) + \gamma \sum_{j=1}^d\xi(w_j) + \frac{\gamma}{\epsilon} \sum_{j=1}^d\|Dw_j\|_1,
\end{equation}
which is of the same form as the TV-model, except for an additional smooth term $\gamma \sum_{j=1}^d\xi(w_j)$. In other word, the same prox-TV algorithm with a slightly modified gradient can be used to find a stationary point for the prox-TV.

\paragraph{Convergence rate.}
Lastly, \citet{reddi2016fast} proved fast convergence rate of proximal SVRG to a stationary point for nonconvex loss functions, which guarantees that with an appropriately chosen learning rate, only $O(1/\varepsilon)$ proximal operators calls and $O(n + n^{2/3}/\varepsilon)$ incremental gradient computation are needed to get to $\varepsilon$ accuracy. 

In our experiments, we find that using diagonal preconditioning with Adagrad \citep{duchi2011adaptive} improves the convergence of both algorithms with little added system overhead; also we find that it is helpful to run a few data passes of prox-SGD before starting prox-SVRG and updating the full gradient only a few data passes (rather than every data pass). Understanding the convergence properties with these additional heuristics included is beyond the scope of this paper.


\section{Experiments}

In this section, we show the accuracy of
the learned latent hazard function evaluated on synthetic data with ground
truth and by conducting case studies on real data evaluated using domain expertise.
We evaluate the out-of-sample predictive power measured
by log-likelihood, which significantly outperforms Cox's models.
The hope is that the hidden hazard function of each attack over
time could reveal the underlying reason why and how some of
the websites were hacked. Finally, we show the generality of the model by applying it to other applications. 

The experimental code is publicly available at
https://github.com/ziqilau/Experimental-HazardRegression.

\begin{figure}[th]
	\centering
	\includegraphics[width=0.45\textwidth,height=0.45\textwidth]{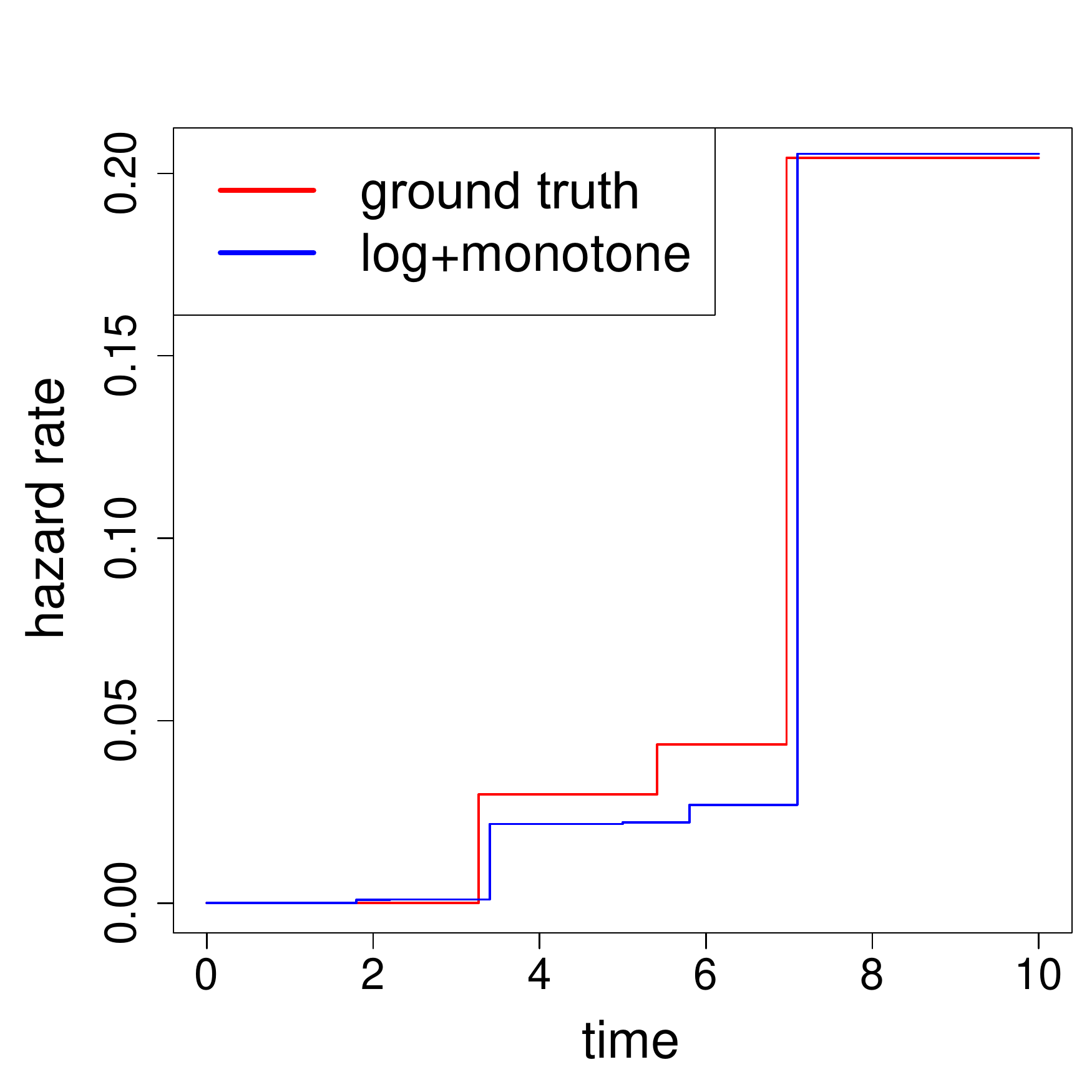}
        \includegraphics[width=0.45\textwidth,height=0.45\textwidth]{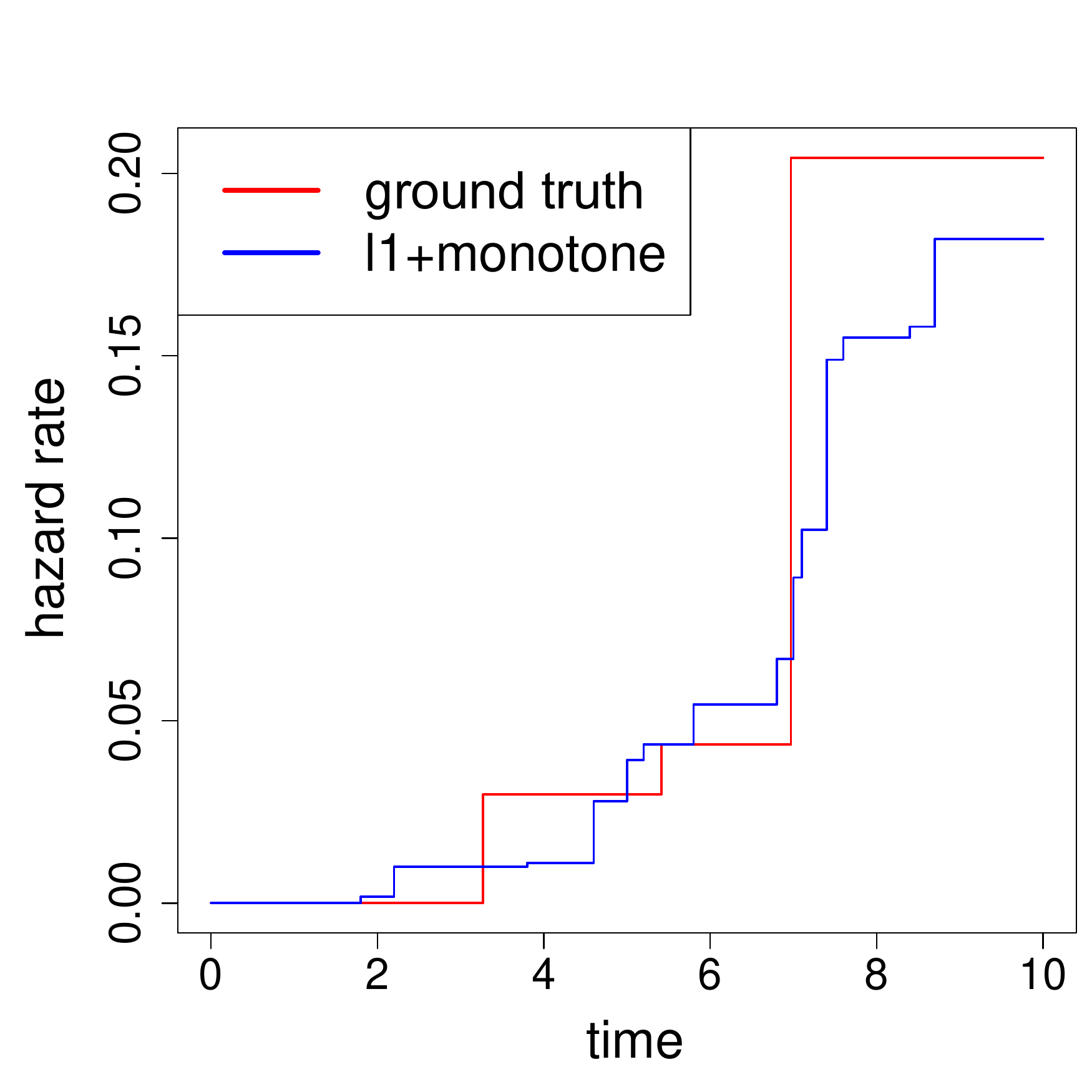}
        \includegraphics[width=0.45\textwidth,height=0.45\textwidth]{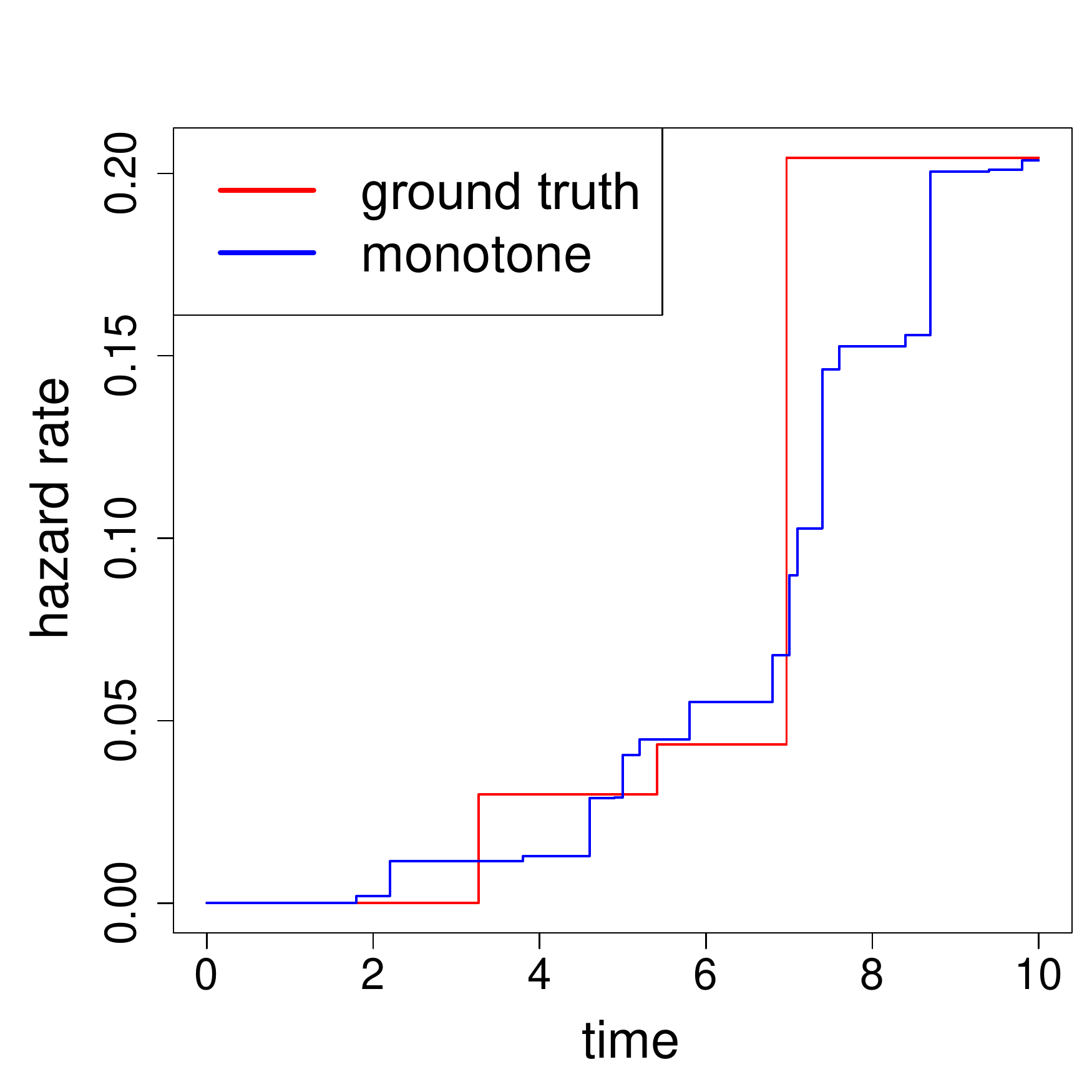}
        \includegraphics[width=0.45\textwidth,height=0.45\textwidth]{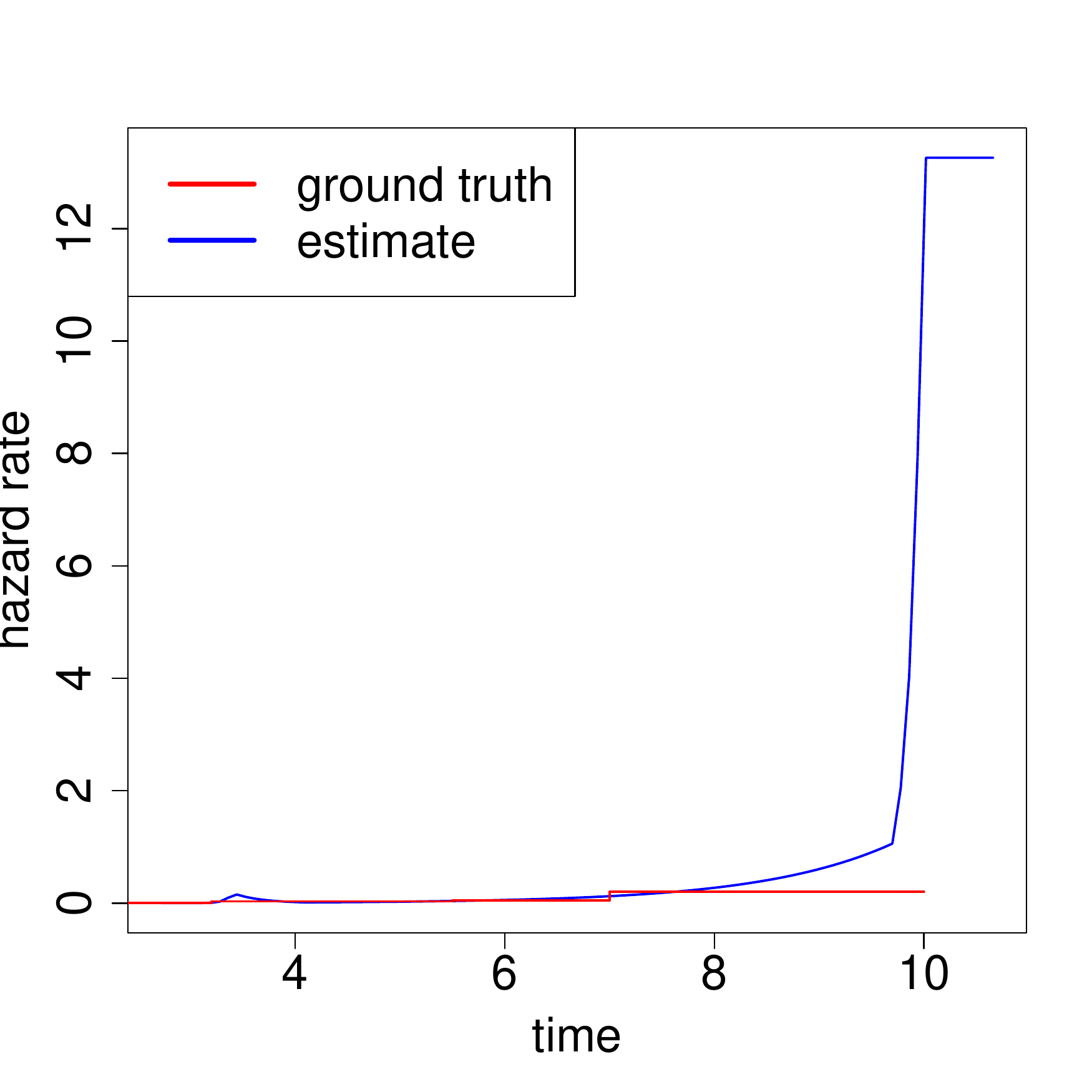}
	\caption{Estimated hazard rate on one exploit: log+monotone({\bf first}), l1+monotone ({\bf second}), monotone ({\bf third} and polspline ({\bf fourth})).}\label{fig:synthetic-case1}
\end{figure}
\begin{figure}[!h]
	\centering
	\includegraphics[width=0.45\textwidth,height=0.45\textwidth]{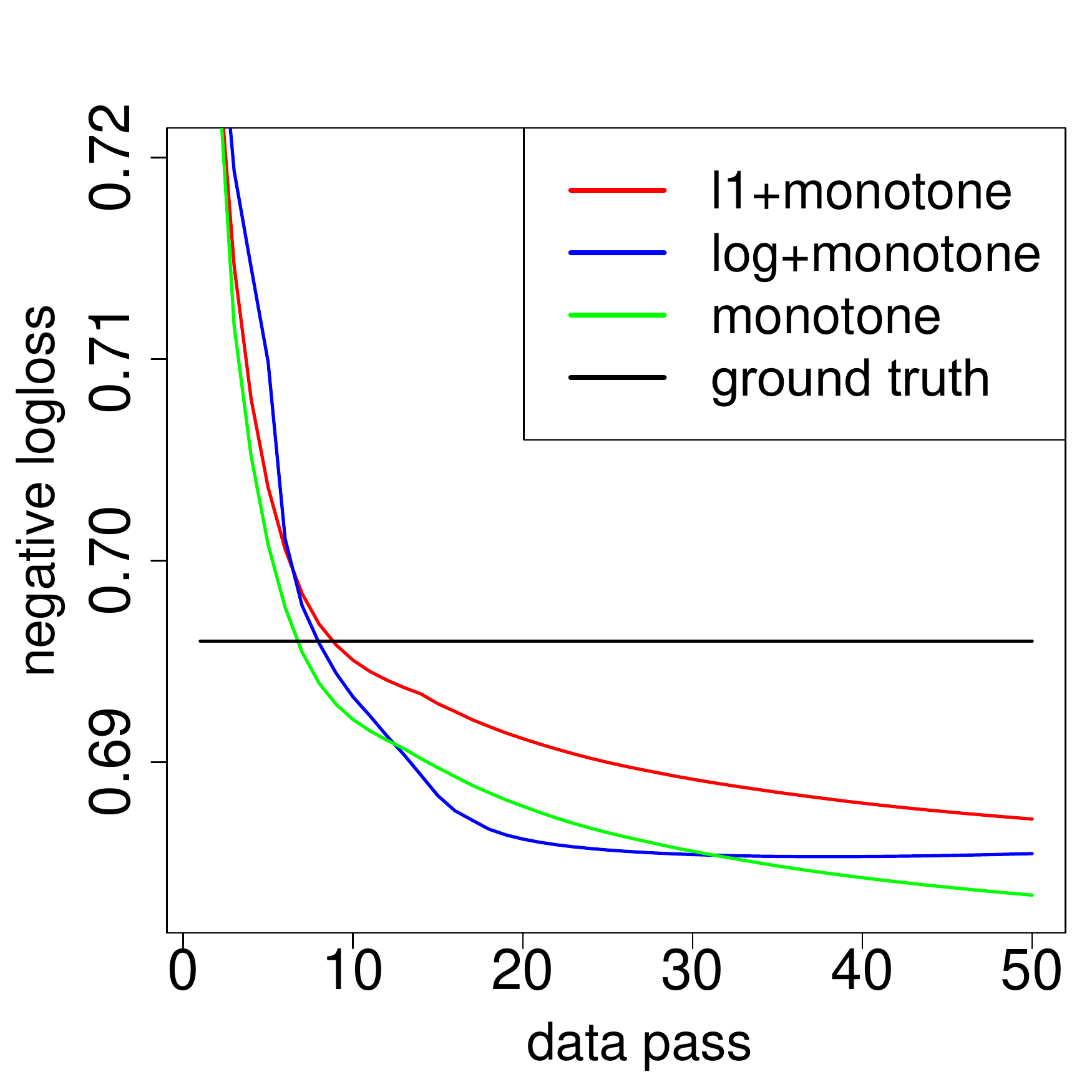}
        \includegraphics[width=0.45\textwidth,height=0.45\textwidth]{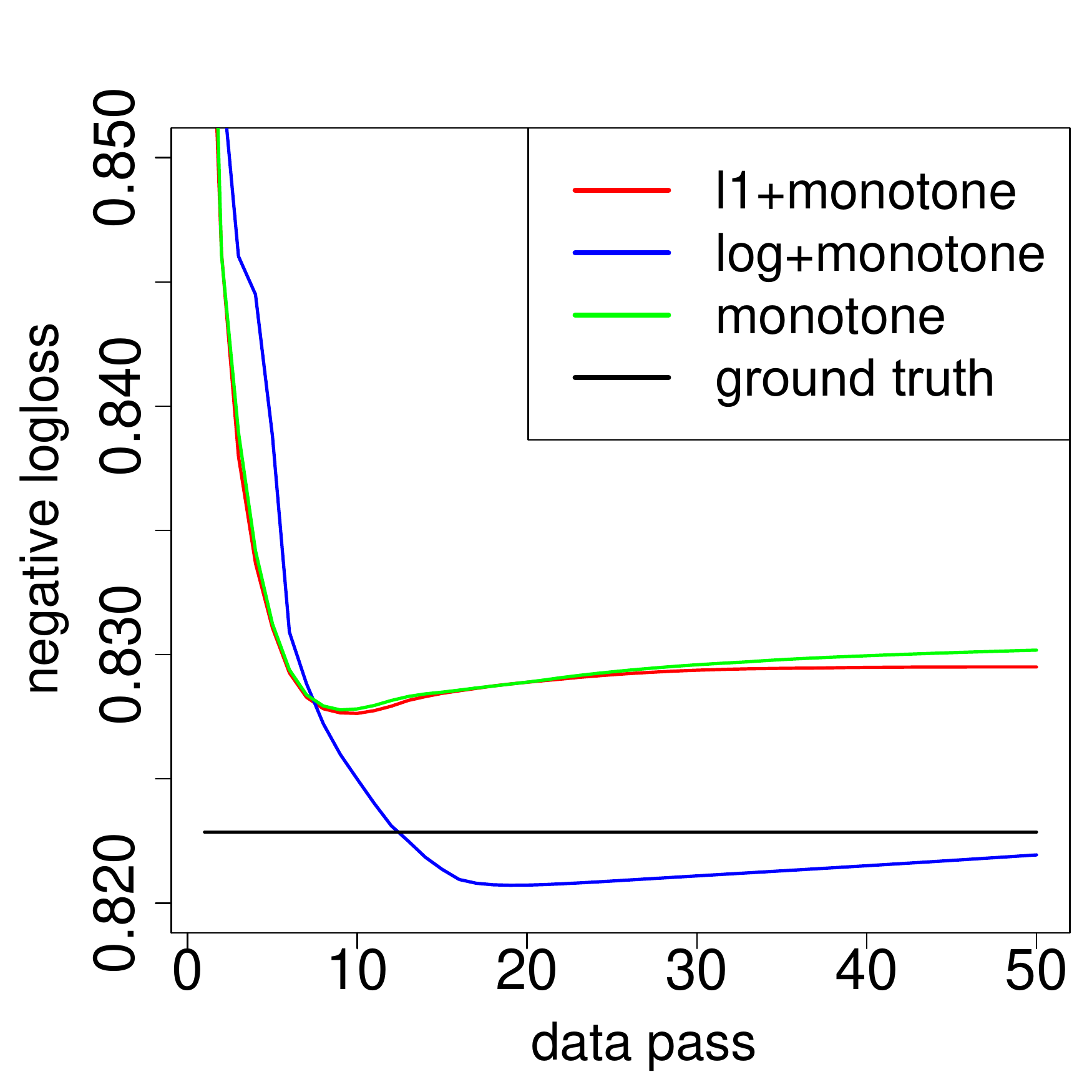}
	\caption{Convergence on training data ({\bf left}) and test data ({\bf right}) respectively. (monotone hazard rate)}\label{fig:synthetic-mono}
\end{figure}
\subsection{Synthetic Experiments}\label{sec:syn}

To demonstrate the effectiveness of the model, we simulate two kinds of attacks. The first type of attack
possess a monotonically increasing hazard rate. This corresponds to
our statistical model with monotone constraint on the hazard rate which is easy to
understand since once a exploit is known it will become
easier and easier for hackers to attack as more tools are available. The second 
type of simulated attack does not have a monotonic hazard rate. 
This leads to our ``non-monotone'' model. It's a practical assumption
because in reality the attack campaigns could be quite complex. 
We will talk about both the
pros and cons of these two schemes in the analysis of the real-world data.

{\bf Data.} In both cases, we simulate the data as follows:
\begin{enumerate}
\item We generate 40 features among which 4 are indicators for the existence 
of a vulnerability that is potentially under attack.
\item To simulate the true attacks, we assume there could be several attack campaigns
for each exploit. For each exploit:
\begin{enumerate} 
\item We randomly pick change points over time, cast as real 
numbers in $[0,10.0]$, each of which corresponds
to the start of one attack campaign. 
\item For each campaign, the hazard rate is randomly sampled.
\end{enumerate}
\item Given the ground true hazard rate we got in step 2, we sample
the exact hacked times for each of 1000 data points.
\item Independently, we assume another
uniformly sampled checking points served as censoring times. Finally we
obtain our experimental interval-censored times by finding the nearest
censoring times around each exact hacked time.
\end{enumerate}

{\bf Comparison methods.} The purpose of this section is to verify the ability of
our models, other time-varying coefficient hazard regression
models and their abilities to recover hazard curves with sharp changes. 
First, most of the existing
state-of-art methods rely on splines where the change points are pre-fixed or automatically
estimated based on model selection. Second, to best of our knowledge, we are not aware of any
existing work allowing both time-varying coefficients and interval censored data.
We study the popular ``polspline''~\cite{Kooperberg1994} which is
a time-varying coefficient hazard regression tool archived in the R
repository\footnote{https://cran.r-project.org/web/views/Survival.html}.
The ``polspline'' automatically learn splines including constant functions,
linear functions of time and so on.
Note that ``polspline'' only works on uncensored data and right censored data.

\begin{figure}[!t]
	\centering
	\includegraphics[width=0.45\textwidth,height=0.45\textwidth]{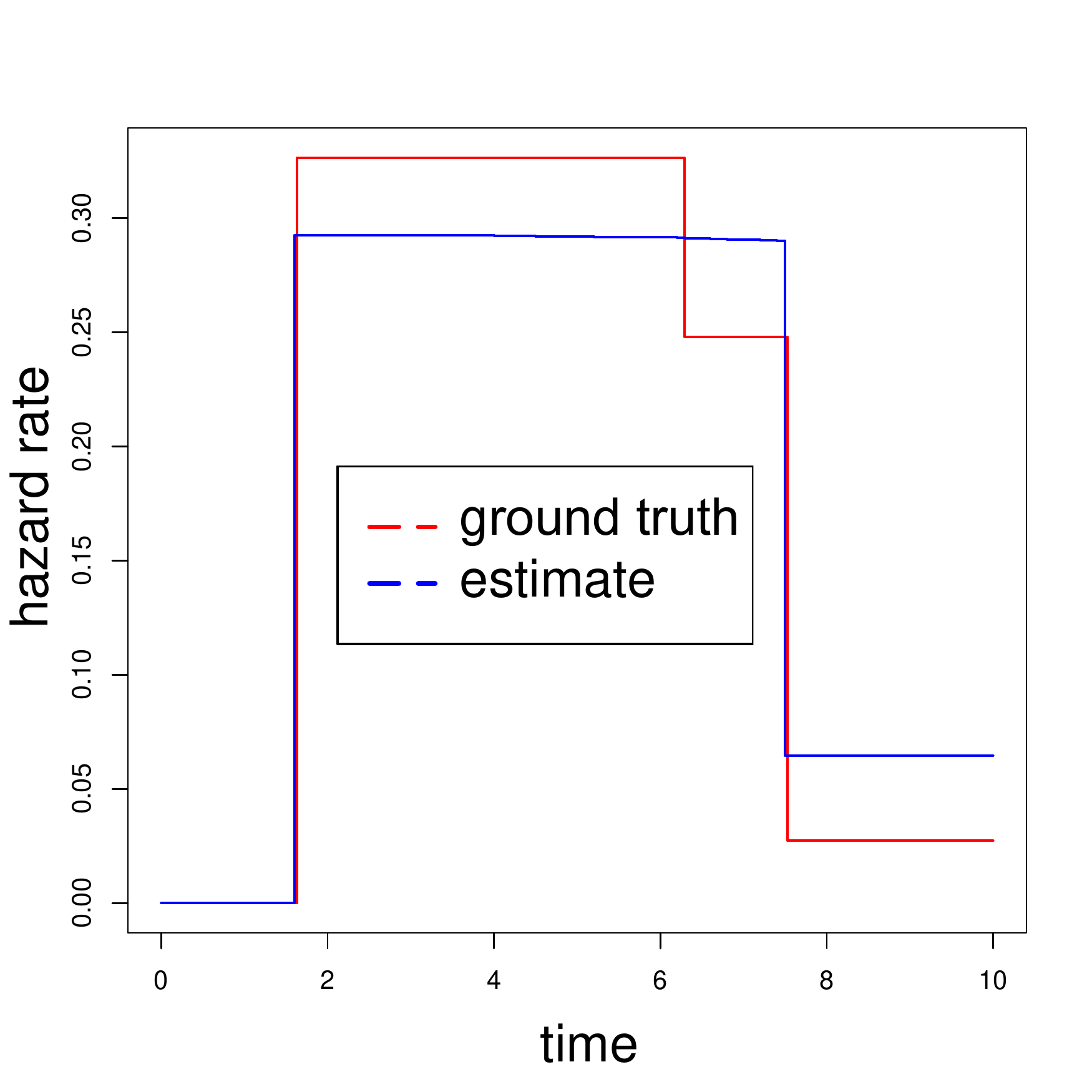}
        \includegraphics[width=0.45\textwidth,height=0.45\textwidth]{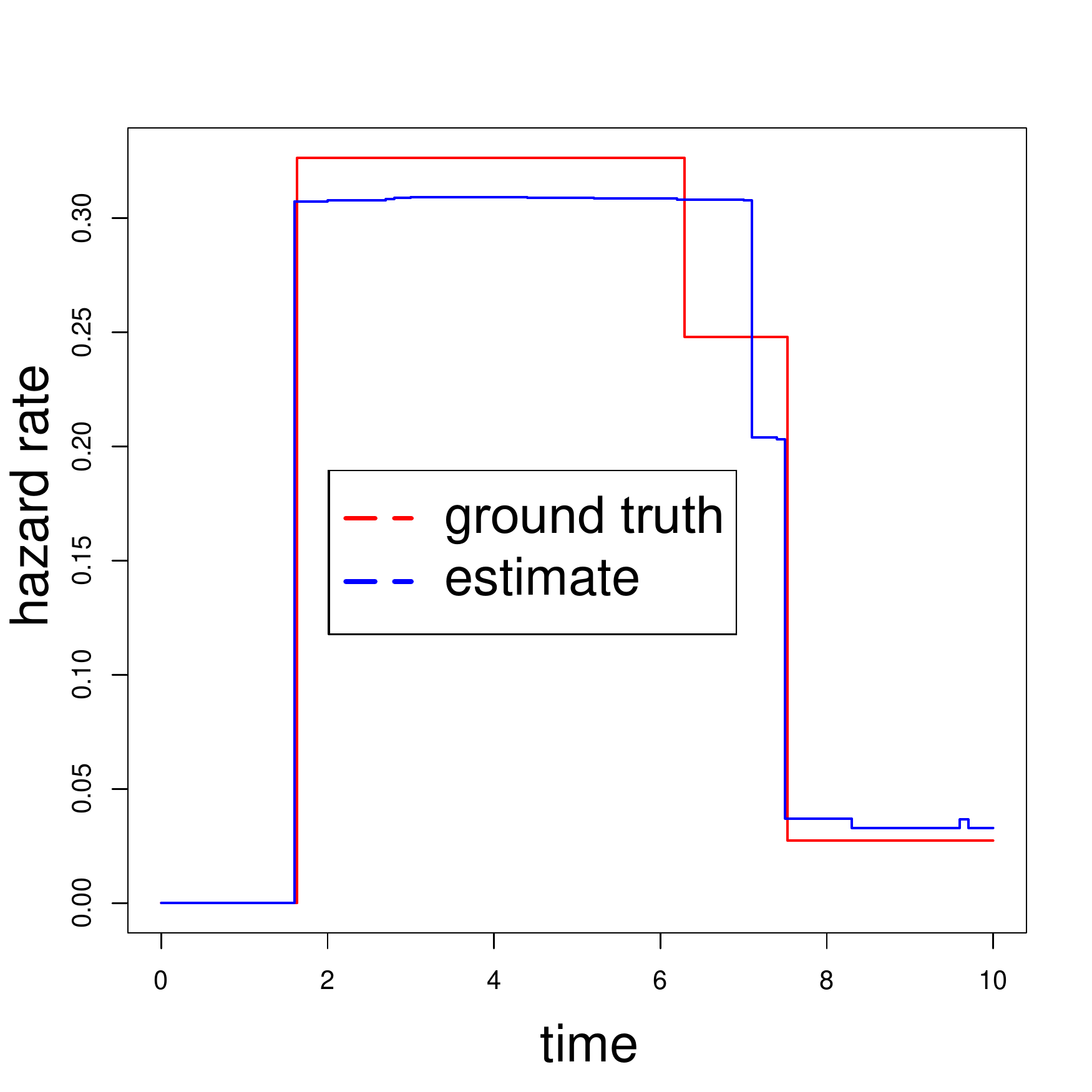}
        \includegraphics[width=0.45\textwidth,height=0.45\textwidth]{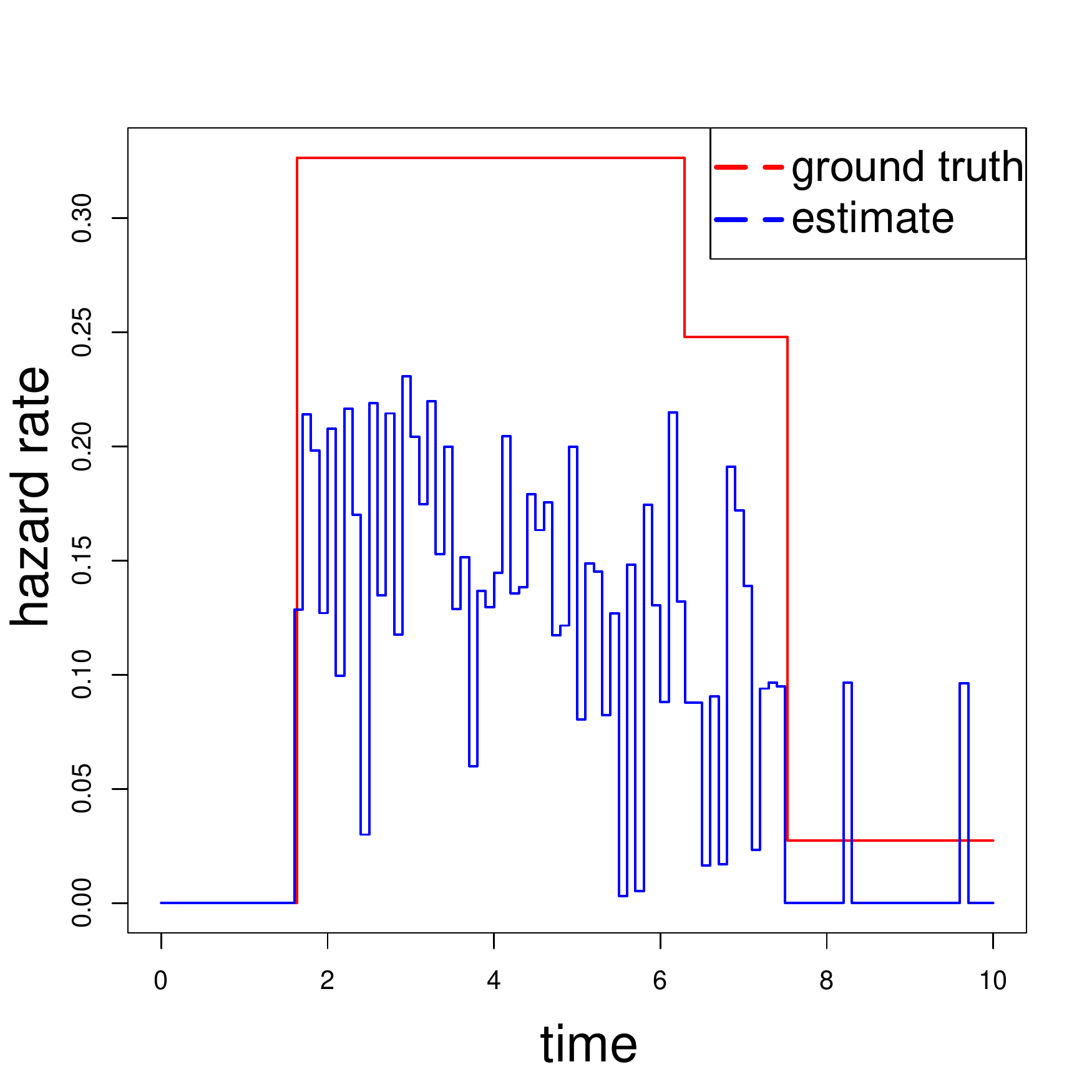}
	\includegraphics[width=0.45\textwidth,height=0.45\textwidth]{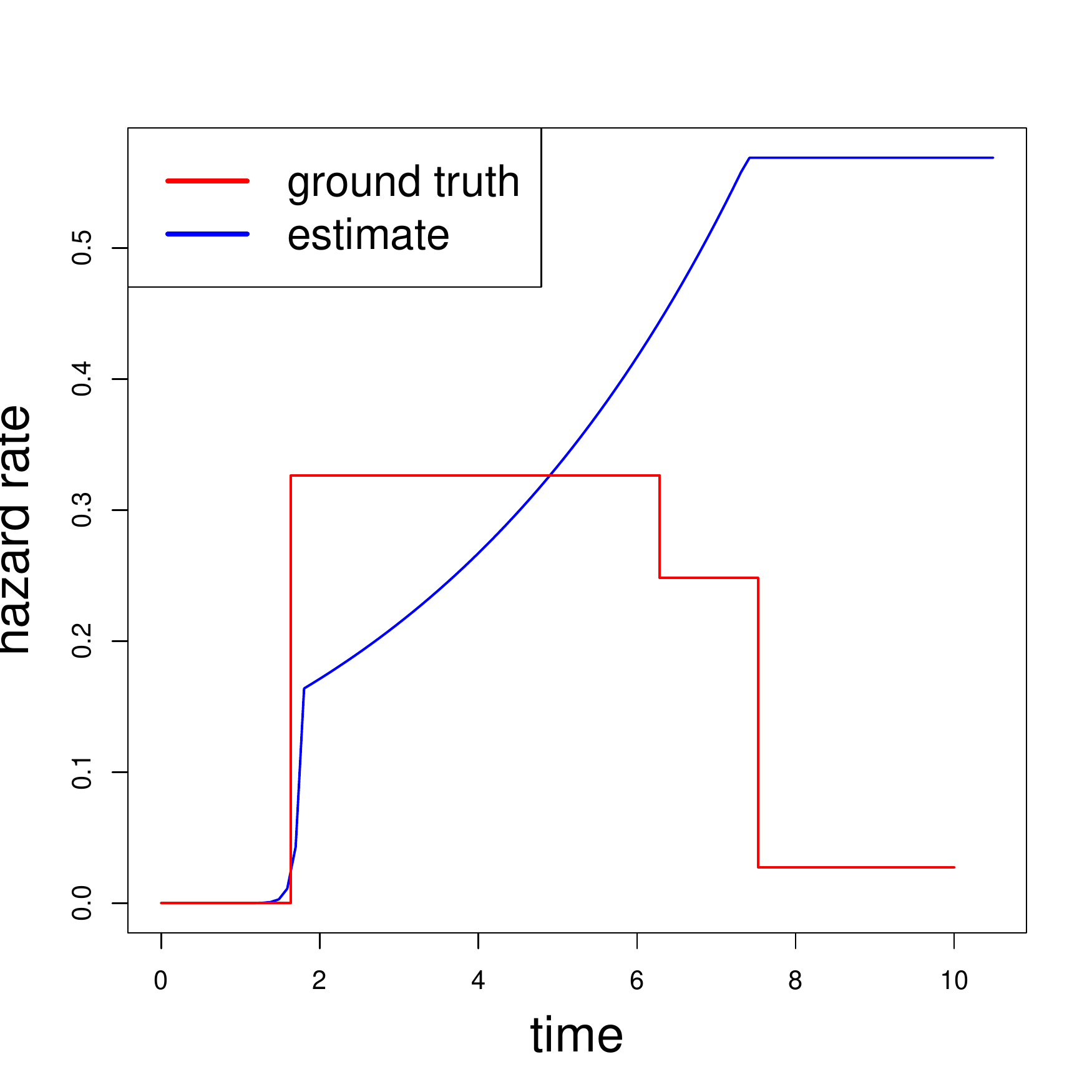}
	\caption{Estimated hazard rate on one exploit: log+l1({\bf first}), l1 ({\bf second}), non-monotone ({\bf third}) and polspline ({\bf fourth}).}\label{fig:synthetic-case2}
\end{figure}
\begin{figure}[!h]
	\centering
	\includegraphics[width=0.45\textwidth,height=0.45\textwidth]{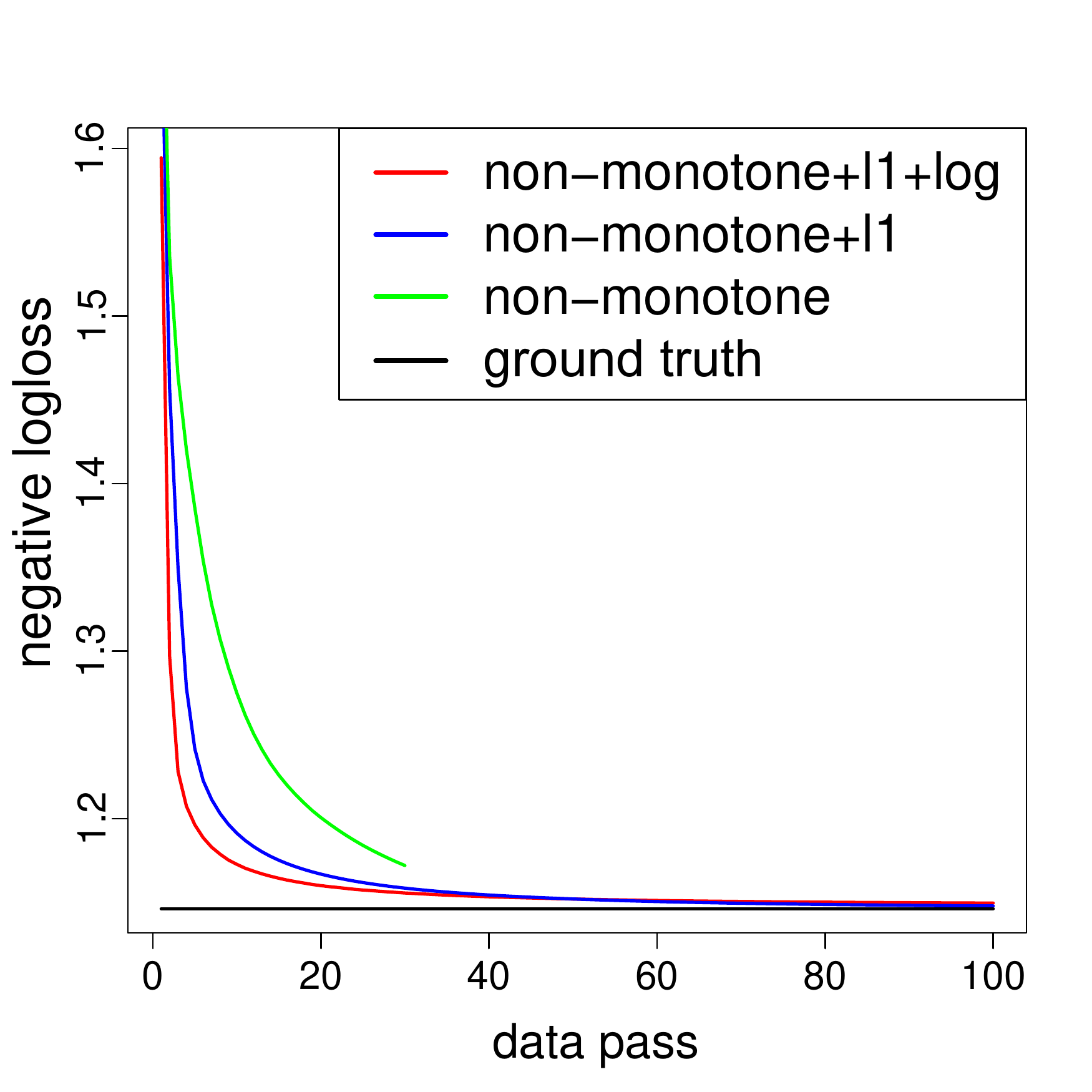}
        \includegraphics[width=0.45\textwidth,height=0.45\textwidth]{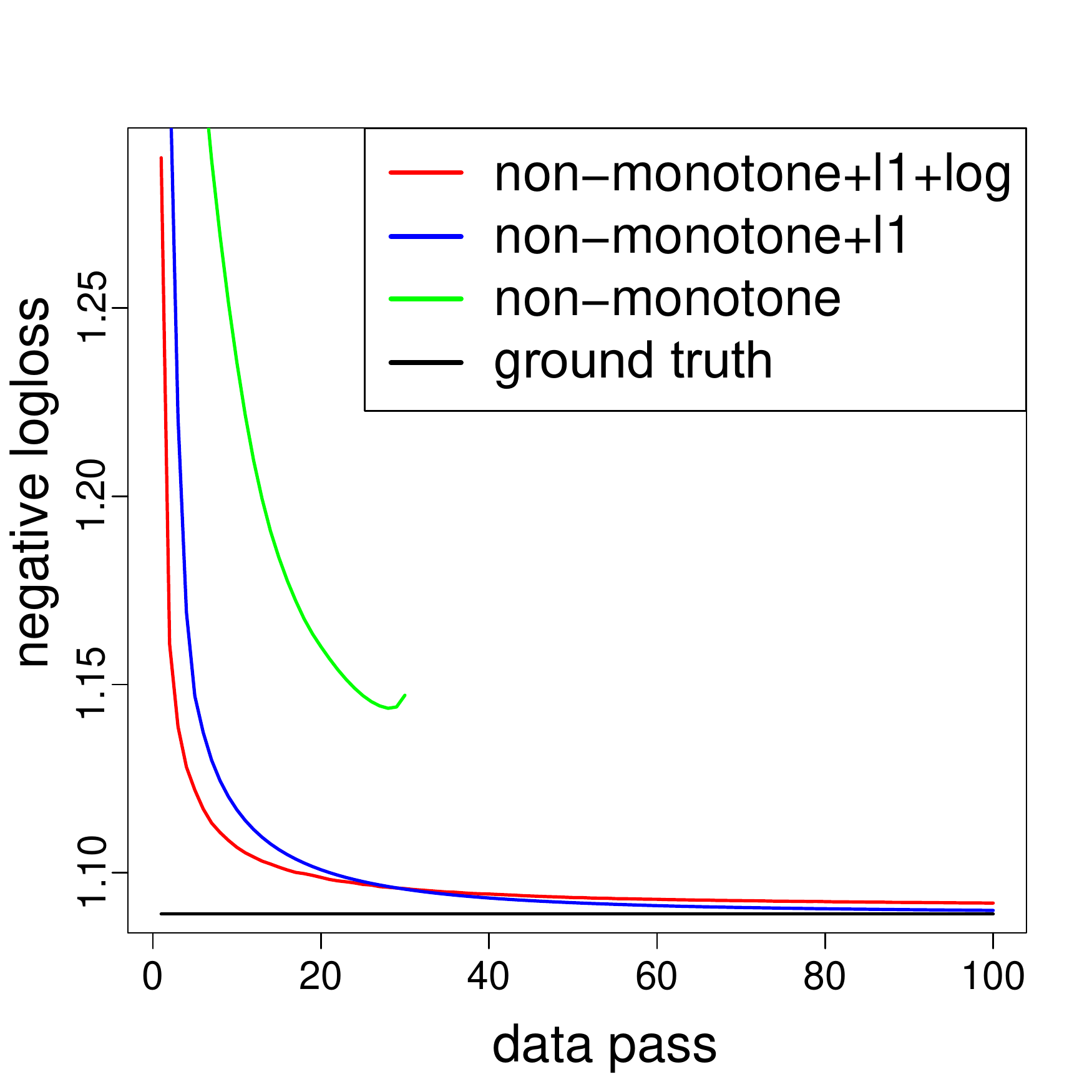}
	\caption{Convergence on training data ({\bf left}) and test data ({\bf right}) respectively. (non-monotone hazard rate)}\label{fig:synthetic-non}
\end{figure}

We denote ``$\ell_1$'' as $\ell_1$ penalized Total Variation, and ``$\log$'' as $\log$ penalty in Eq.\eqref{eq:discrete-log}.
For convenient, we use the term ``non-monotone'' as standard model in Eq.\eqref{eq:fullprox} without Total Variation or log penalty.
Similarly we use the term ``monotone'' as monotone model in Eq.\eqref{eq:fullprox} without Total Variation or log penalty.
In our experiments, we will see the effects of placing ``$\ell_1$'' and ``$\log$'' onto ``monotone'' and ``non-monotone'' models.

The results for monotonic hazard rates are reported in Figure~\ref{fig:synthetic-case1} and \ref{fig:synthetic-mono}.  The convergence in Figure~\ref{fig:synthetic-mono}
shows that compared with ``$\ell_1$'' and monotone, ``$\log$'' penalty works a bit better. The reason for this can be seen from Figure~\ref{fig:synthetic-case1} (1 out of 4
exploit) where
the ``$\log$'' penalty produces a much sharper hazard curve and
approximates the ground truth quite well.

Figures~\ref{fig:synthetic-case2}~and~\ref{fig:synthetic-non} and  show the results on data generated without the monotonic hazard rate constraint. In this case both the
``$\ell_1$'' and ``$\log$'' penalties work well. It is expected that the non-monotone model without any regularizer will overfit the data quite aggressively. The minor
difference between ``$\ell_1$'' and ``$\ell_1+\log$'' is that
``$\ell_1+\log$'' produces sharper curves but tends to ignore weak
signals, e.g.\ the second knot, when the signal-to-noise ratio is
relative small, i.e.\ it prefers significant signals.
The convergence in Figure~\ref{fig:synthetic-non}
shows that both ``$\ell_1$'' and ``$\log$'' penalty significantly
outperform the plain non-monotone model. 

Note that our model consistently assigns all features that are not exploits to zero.

In order to compare with our models, we build the data for ``polspline'' as follows.
First, we start with the interval and right censored data simulated in
the beginning of this section. We leave right censored data unchanged.
Second, we generate two data points (uncensored and right censored) for every
interval censored data point, i.e., the left end point of interval censored
time serves as right censored time, and
the exact hacking time serves as uncensored time.

We show the results of ``polspline'' in Figures~\ref{fig:synthetic-case1}
and \ref{fig:synthetic-case2}.
It can be seen that spline based models which allow relatively
smooth changes struggle to model the sharp hazard rate well enough. It is expected from the shape of the hazard
rate curve that the negative log-likelihood of ``polspline'' compared with our models is
inferior. Due to the algorithm of ``polspline'' is not iterative-like, we are not able to
obtain the results of estimated hazard rate in each iteration, therefore we instead
report the final results after the algorithm is done in Table~\ref{tb:pol}.
\begin{table}[tb]
	    \begin{tabular}{|c|c|c|c|}
	    	\hline
	    	Cases & Ground Truth & Our models & polspline\\\hline
	    	non-monotone & 1.082 & 1.085 & 4.123\\
	    	monotone & 0.823 & 0.823 & 3.512\\\hline
	    \end{tabular}
  \centering
  \caption{\label{tb:pol} Negative log-likelihood of the estimates of ``polspline'' compared with our models on test set, synthetic settings.}
\end{table}

\subsection{Real-World Data}

The data used for evaluation was sourced from the work of \citet{Soska2014} and was comprised as a collection of interval censored sites from blacklists and right censored sites randomly sampled from .com domains\footnote{A .com zone file is the list of all registered .com domains at the time.}. As a consequence of the time-varying distribution of software deployed on the web, all the samples were drawn from The Wayback Machine\footnote{The Wayback Machine is a service that archives parts of the web.} only when archives were available at appropriate dates.

One of the blacklists that was sampled was PhishTank, a blacklist of predominately phishing\footnote{A phishing website is a website that impersonates another site such as a bank, typically to trick users and steal credentials.} websites for which 11,724,276 unique URLs from 91,155 unique sites were observed between February 23, 2013 and December 31, 2013. The Wayback Machine contained usable archives
for {\em 34,922} (38.3\%) domains. The other blacklist that was used contains websites that perform search redirection attacks \cite{Leontiadis2014} and was sampled from October 20, 2011 to September 16, 2013. In total the sample contained 738,479 unique links, from 16,173 unique domains. The Wayback Machine contained archives in the acceptable range for {\em 14,425} (89\%) of these sites.

These two blacklists are particularly well suited for providing labeled samples of attacked websites as manual inspection has shown, an overwhelmingly large proportion of these sites were compromised by a hacker. This is contrary to other websites which may simply have been maliciously hosted or contained controversial content without even being vulnerable or hacked. 

Lastly, the .com zone file from January 14th, 2014 was randomly sampled, ignoring cases where an image of the site was not available in The Wayback Machine. In total {\em 336,671} archives distributed uniformly between February 20th, 2010 and September 31st, 2013 were collected. These samples were checked against our blacklists as well as Google Safe Browsing to ensure that as few compromised sites remained in the sample as possible.

We automatically extracted raw tags and attributes from webpages, that served as features (a total of {\em 159,000} features). Examples of these tags and attributes include <br>, and <meta> WordPress 2.9.2</meta>. They are useful for indicating the presence of code that is vulnerable or may be the target of adversaries. Our corpus of features corresponds to a very large set of distinct and diverse software packages or content management systems.

\subsection{Real-World Numeric Results}\label{sec:numeric}
There are a total of 120 million websites registered in .com zone file at the end
of our observation. According to a rough estimates of the distribution
of hacked and non-hacked sites, 
we reweigh the non-hacked websites by 200 times. To report the
results, we randomly select 80\% for training and validation, and the rest as test data.

{\bf Comparison methods.} The baseline method is the classic Cox Proportional model~\citep{Cox1972} which has been extensively used for
hazard regression and survival analysis. Ever since its invention, has been considered a ``gold standard'' in epidemiology, clinical trials and biomedical studies \citep[see e.g.,][]{woodward2013epidemiology}. The Cox model is parametrized based on the 
features just as our model is, but is not time-varying. 
As has been discussed in section~\ref{sec:stat}, to estimate
the survival probabilities we specify a uniform distribution for 
the baseline hazard function. We are not
able to compare with other time-varying models in R repository
due to the scale of the data.

An experimental comparison between our models and the Cox model on the aforementioned dataset are shown in Figure~\ref{fig:real-convergence}. By comparison, the 
Cox model underfits the data quite a bit. Our ``monotone'' model which allows only a non-decreasing hazard rate also underfits the data slightly but still significantly outperforms Cox. Moreover due to much smaller
parameter space need to search, we find that it converges faster than the ``$\ell_1$+non-monotone'' model. Additionally, our ``$\log$+monotone'' model performs nearly
the same on convergence (overlapped). Again it is well expected that ``non-monotone'' model without any constraint overfits the data severely.  As a consequence, the ``$\ell_1$+non-monotone''
model which is well-regularized performs the best. These results clearly
show that the latent hazard rate recovered by our models improves upon the classic Cox model in this setting.

Due to the sparsity of our models, table~\ref{tb:size} shows that we require only about three times the storage to give significantly better
estimates compared with the Cox model. Most importantly, identifying the changes of each feature's susceptibility over time can help people understand the latent hacking campaigns and leverage these insights to take appropriate action. We will discuss more in section~\ref{sec:case}.
\begin{figure}[tb]
	\centering
	\includegraphics[width=0.45\textwidth,height=0.45\textwidth]{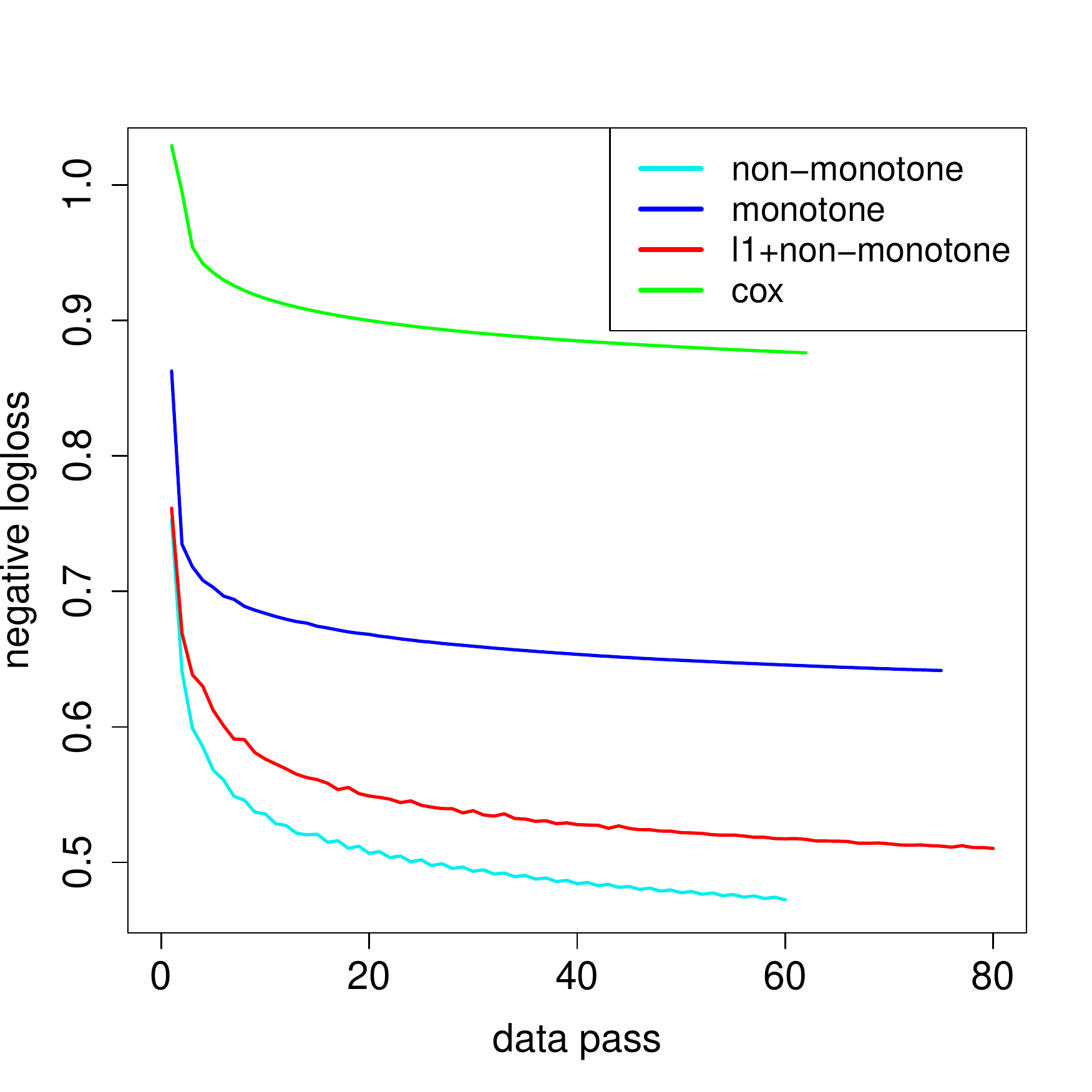}
	\includegraphics[width=0.45\textwidth,height=0.45\textwidth]{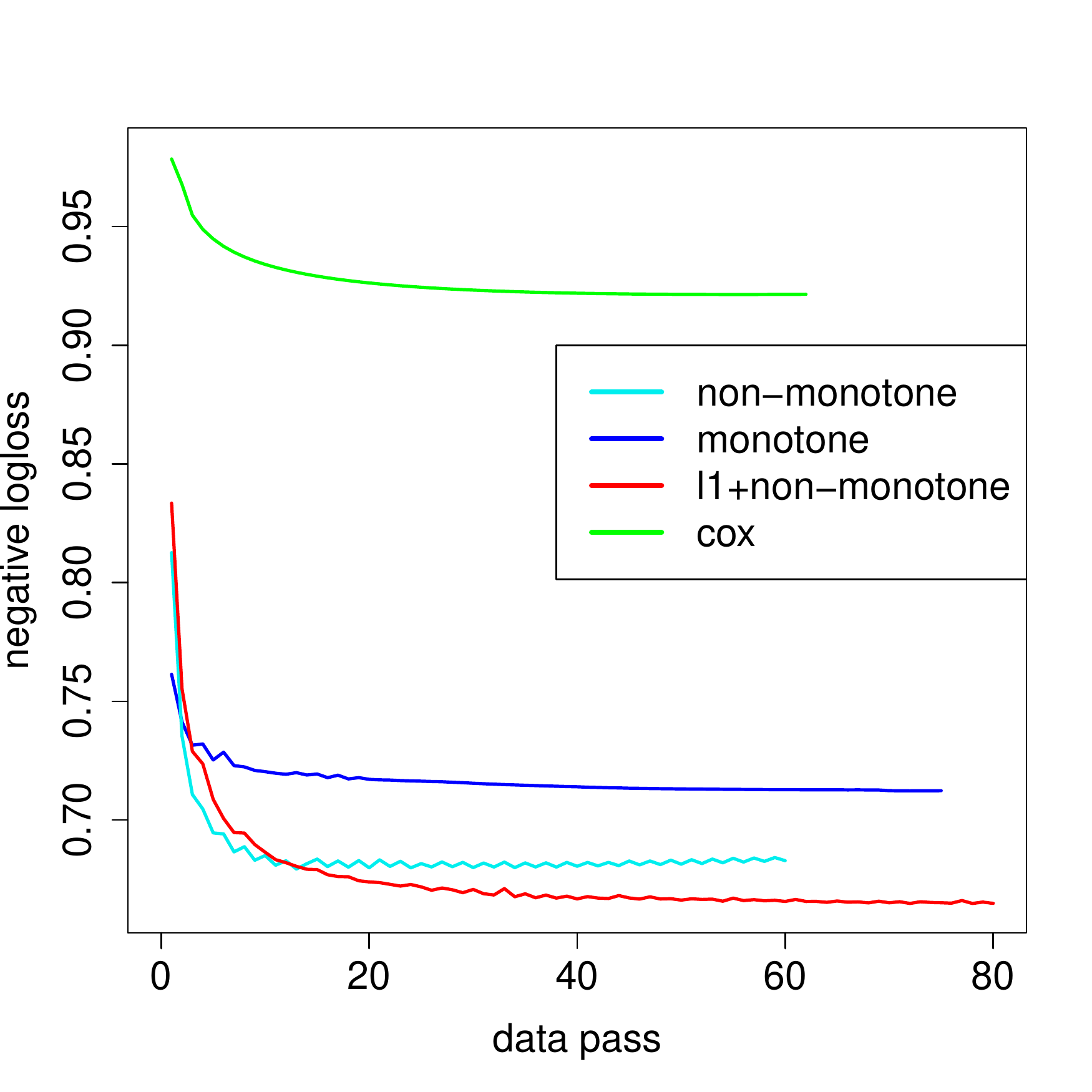}
	\caption{Convergence on training data ({\bf left}) and test data ({\bf right}) respectively.}\label{fig:real-convergence}
\end{figure}
\begin{table}[tb]
	    \begin{tabular}{|c|c|}
	    	\hline
	    	Methods & Empirical model size\\\hline
	    	non-monotone & $2\cdot10^6$ \\
	    	monotone &$4.04\cdot 10^5$\\
	    	 $\ell_1$+nonmonotone & $5.16\cdot 10^5$ \\
	    	 Cox  &$1.59 \cdot 10^5$ \\\hline
	    \end{tabular}
  \centering
  \caption{\label{tb:size} Empirical model size (active breakpoints for our methods, number of parameters for Cox) estimated by different statistic models.}
  \end{table}


Finally it is imperative that the model does not assign non-zero hazard rate
to features that are uncorrelated with the security outcome of a website.
The hazard curve for 200 random features believed to be uncorrelated
with security (such as code for custom font colors, styles, and links
to unique images) were manually studied, 182 (9\% false positive rate)
of which generated a hazard value of 0 for the entire duration of
the experiment. Of the 18 features that were assigned a non-zero
hazard curve, all of them reported a value of less than 0.04 which
can be ignored.

\subsubsection{Real-World Case Study}\label{sec:case}



\begin{figure}[tb]
	\centering
		\includegraphics[width=0.45\textwidth, height=0.45\textwidth]{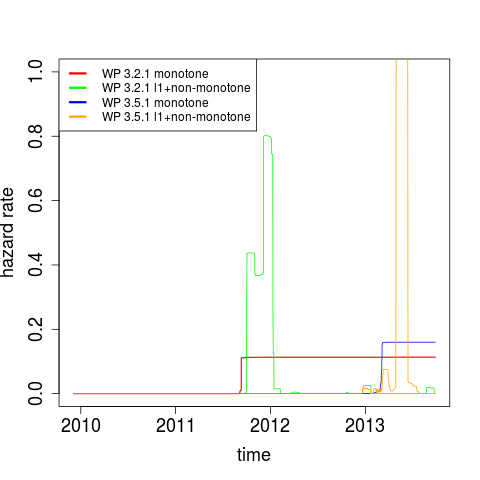}
		\includegraphics[width=0.45\textwidth, height=0.45\textwidth]{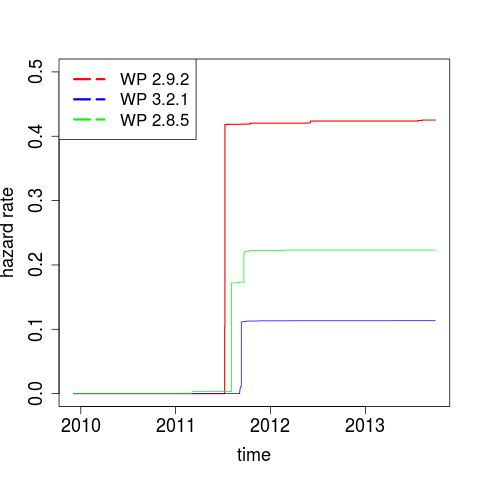}
	\caption{$\lambda_t(i)$ of a feature known to correspond directly to instances of Wordpress 3.2.1 and Wordpress 3.5.1.({\bf left}); $\lambda_t(i)$ of features known to correspond directly to different versions of the Wordpress content management system that were attacked in the summer of 2011.({\bf right})}
	\label{fig:wp321351}
\end{figure}

In this section, we manually inspect the model's ability to automatically discover
known security events. To this end, the model was trained on the aforementioned
dataset and $\lambda_i(t)$  and was measured for features $i$ that corresponded
directly to websites that were known to be the victim of attacks.

Figure~\ref{fig:wp321351} ({\bf left}) demonstrates some of the differences between the well penalized monotone and non-monotone models by following the hazard assigned to features that correspond to Wordpress 3.5.1. In early 2013, our dataset recorded a few malicious instances of Wordpress 3.5.1 sites (among some benign ones). These initial samples appeared to be part of a small scale test or proof of concept by the adversary to demonstrate their ability to exploit the platform. Both models detect these security events and respond by assigning a non-zero hazard.

Following the small scale test was a lack of activity for a few weeks, during which the non-monotone model relaxes its hazard rate back down to zero, just before an attack campaign on a much larger scale is launched. This example illustrates the importance of not letting a guard down in the context of security. Once a vulnerability for a software package is known, that package is always at risk, even if it is not actively being exploited. 

\begin{figure}[!ht]
	\centering
		\includegraphics[width=0.45\textwidth]{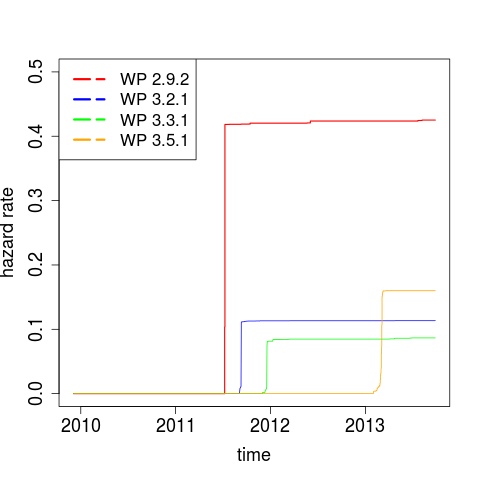}
	\caption{$\lambda_t(i)$ of features known to correspond directly to particular versions of Wordpress.}
	\label{fig:wppaper}
\end{figure}

Despite not taking the most prudent approach to security, the non-monotone model captures the notion that adversaries tend to work in batches or attack campaigns. Previous work~\cite{Soska2014} has shown that it is economically efficient for adversaries to compromise similar sites in large batches, and after a few attack campaigns, most vulnerable websites tend to be ignored. This phenomena is shown in Figure~\ref{fig:wp321351} where Wordpress 3.2.1 was attacked in late 2011 and then subsequently ignored with the exception of a few small attacks that were likely the work of amateurs or password guessing attacks which are orthogonal to the underlying software and any observable content features. The monotone model in this case is very prudent while the non-monotone model captures the notion that the software is not being targeted.

It can be observed from Figure~\ref{fig:wp321351}~({\bf right}) that a number of distinct  Wordpress distributions experienced a change-point in the summer of 2011 (between July 8th 2011 and August 11th 2011). This phenomena was present in several of the most popular versions of Wordpress in the dataset including versions 2.8.5, 2.9.2 and 3.2.1.

This type of correlation between the hazard of features corresponding to different versions of a software package is expected. This correlation often occurs when adversaries exploit vulnerabilities which are present in multiple versions of a package, or plugins and third party add-ons that share compatibility across the different packages.

Manual investigation revealed that a number of impactful CVEs\footnote{CVE stands for Common Vulnerabilities and Exposures, which is a list of publicly disclosed vulnerabilities and security risks to software.} such as remote file inclusion and privilege escalation were found for these versions of Wordpress as well as a particular plugin around the time of July 2011. While it is impossible to attribute with certainty any particular vulnerability, the observed behavior is consistent with vulnerabilities that impact large number of consecutive iterations of software.

Another spot check for the model is the ability to corroborate existing literature on malicious web deteciton. Figure~\ref{fig:wppaper} demonstrates the change-points in $\lambda_i(t)$ for specific versions of Wordpress. The model assigns Wordpress 2.9.2, 3.2.1, 3.3.1 and 3.5.1 change-points around July 2011, August 2011, December 2011, and February 2013 respectively. The work of Soska et al.~\cite{Soska2014} found nearly identical attack campaigns for Wordpress 2.9.2, 3.2.1 and 3.3.1 but failed to produce a meaningful result for 3.5.1.

\subsection{Studies on Dropout rate from Alipay.com}\label{sec:alipay}
\begin{figure}[!ht]
	\centering
		\includegraphics[width=0.45\textwidth]{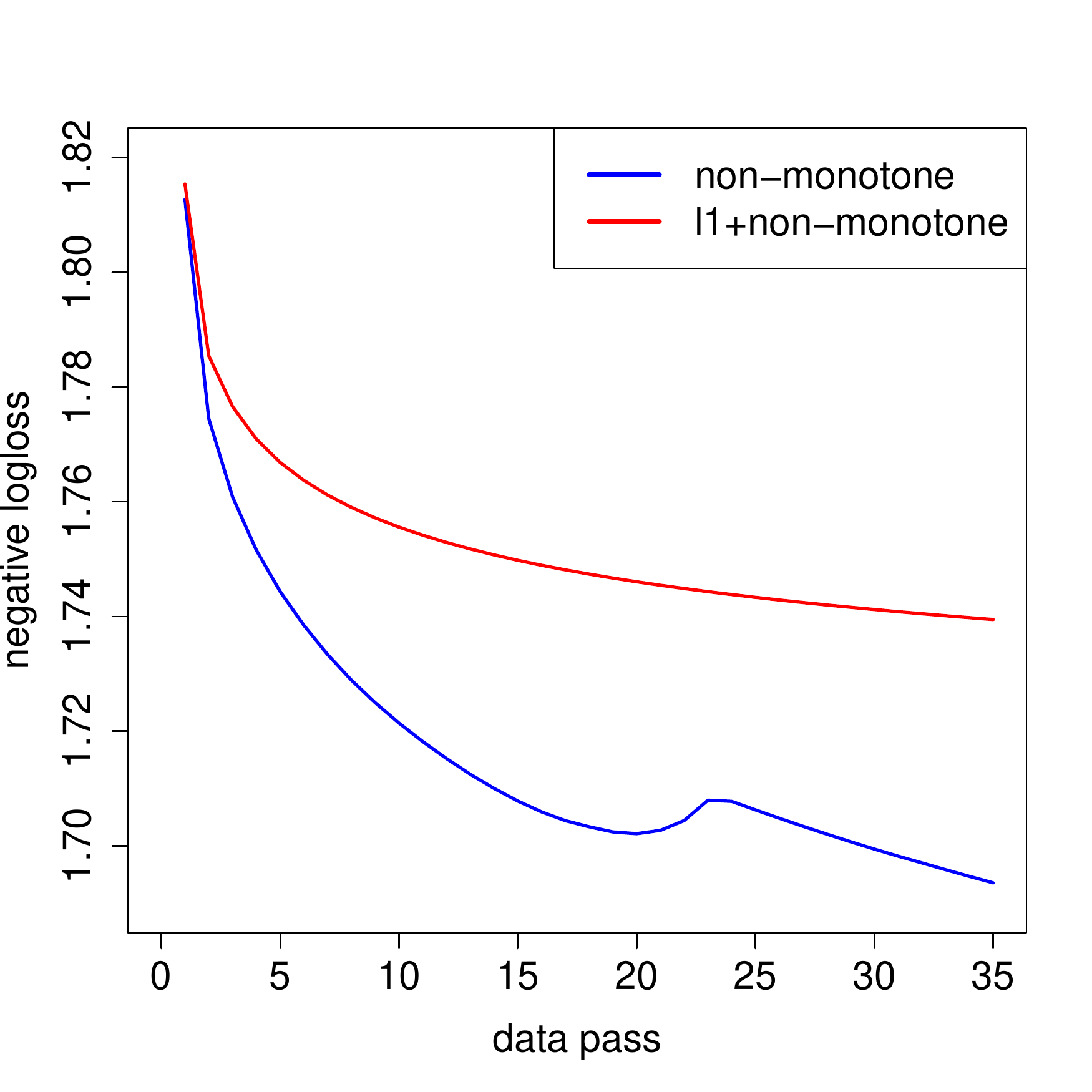}
		\includegraphics[width=0.45\textwidth]{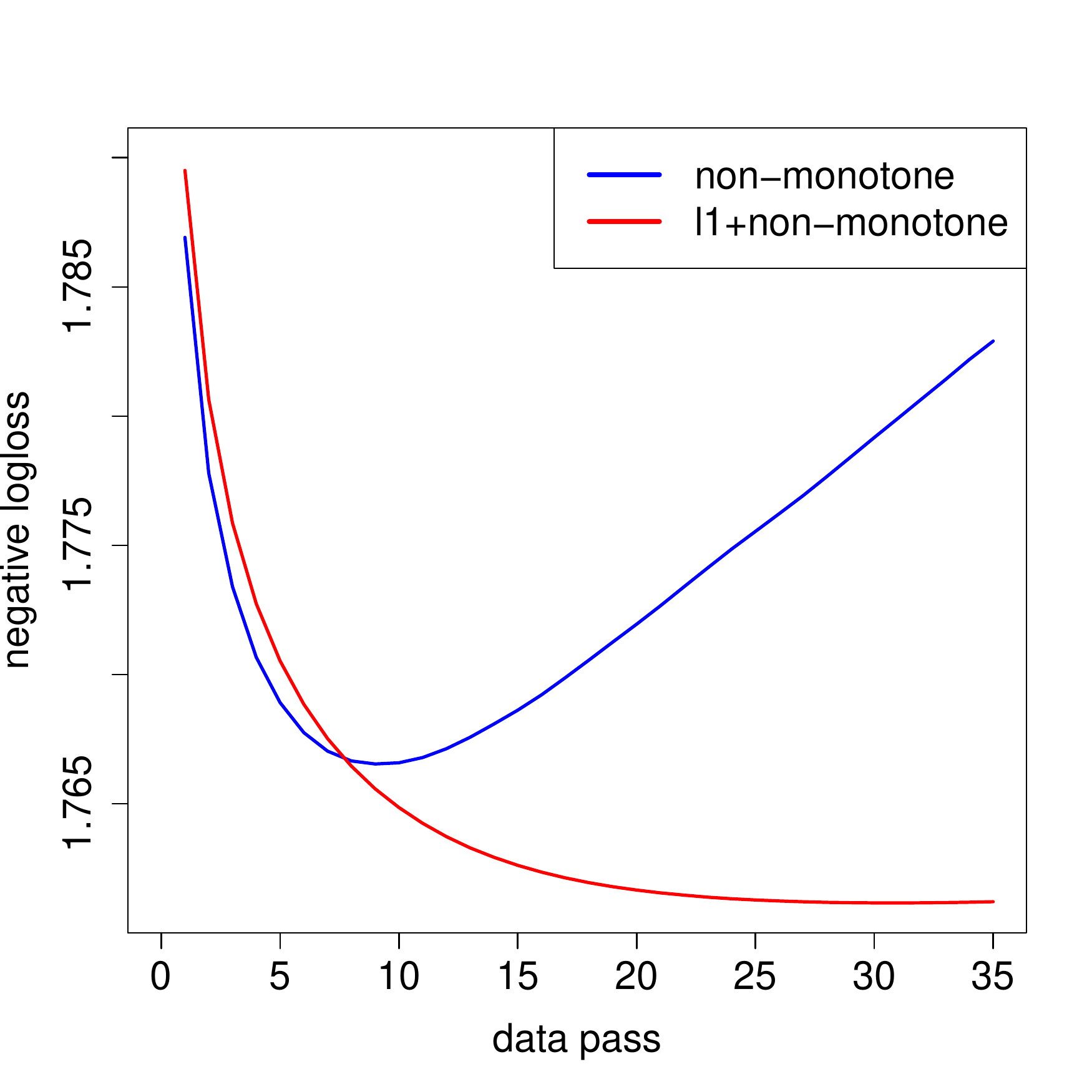}
	\caption{Convergence on Dropout data from Alipay.com: convergence on train {\bf (left)}, covergence on test {\bf (right)}.}
	\label{fig:dropout_conv}
\end{figure}
\begin{figure}[!ht]
	\centering
		\includegraphics[width=0.45\textwidth]{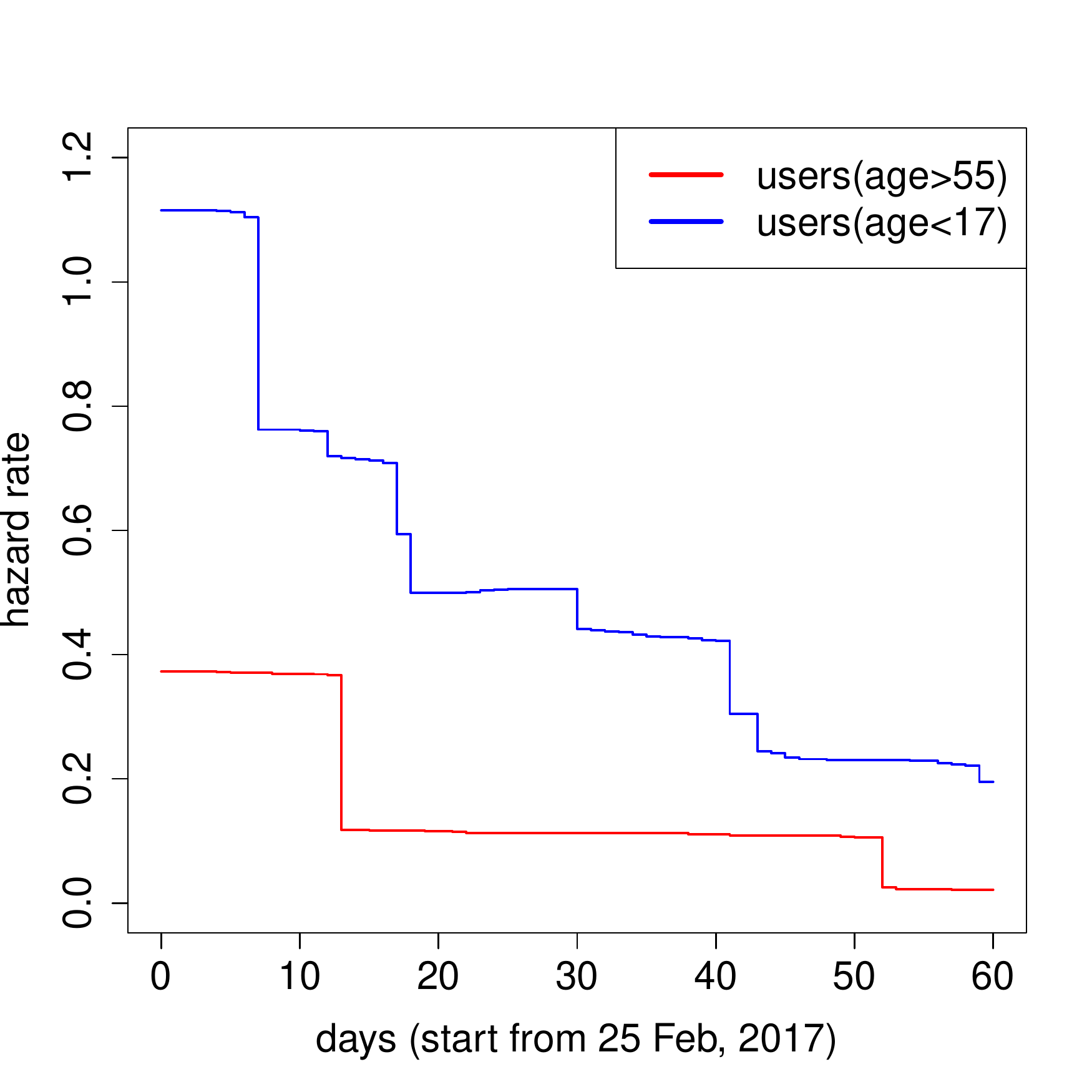}
	\caption{Hazard rate on different features related to the ages of users.}
	\label{fig:dropout_case}
\end{figure}
It is worth noting that our time-varying hazard regression model with capacity of
adaptively learning local estimates of hazard rate could be potentially useful
in other scenarios but not limited in analyzing hacking procedures. In this section,
we study another dataset from Alipay\footnote{https://www.alipay.com/}, a third-party
online payment platform that serves 400 million users in China and control of
half of China's online payment market in October 2016. Basic services provided
by ``Alipay app'' include the digital wallet and personalized investment through
various commercial funds. Here, the dropout rate of users and how does the rate
respond to campaigns promoted by Alipay is the object of interest.

We define one ``dropout'' as a user will not login in seven consecutive days.
The Alipay service would simulate or promote new campaigns for the users who
are apt to dropout.

We randomly select 3,760,455 users who were active during 20 and 26 Feb as long as
the user login Alipay app in any of these days. There are around 8 million
features representing users' demographic background and past behaviors. We observe
the users' dropout behaviors during 20 Feb and 30 Apr. The ratio of users who survive
until 30 Apr is a confidential number due to commercial privay.

We analyze the data with our ``non-monotone''
models to study the dropout rate and how does it change as new campaigns are promoted
. The convergence rate on training data and test data are shown in
\ref{fig:dropout_conv}. No surprisingly, the ``non-monotone'' model without any
constraint overfit the data severely. The ``l1+non-monotone'' model with $\|Dw\|_1$
generalized well.

It is interesting to study the curves of hazard rate associated with each feature.
We show examples in Figure~\ref{fig:dropout_case} that how the estimated hazard
rates change over time. The overall trend of users who are older than 55 and
youth were declining. We align some change points with important campaigns launched
at Alipay. On 10 Mar, Alipay announced a brand new program that every proper payment
are eligible for a cashback. On 18 Apr, Alipay announced another brand
new healthcare program that each payment can accumulate users' own health insurance
amount. The two change points along ``red'' line in figure~\ref{fig:dropout_case}
respond to such two campaigns quite clearly. On the other hand,
we observed a sharp decline around 5 March among
youth users, this is probably because of the consecutive warm-up campaigns
for 8 March shopping
festival. That means the youth are interested in consuming rather
than saving compared with the habits of older users. Such analyses would be
useful for understanding the strategies
of promotion in the near future.

\section{Conclusion}
In this paper, we propose a novel survival analysis-based approach to model the
latent process of websites getting hacked over time. The proposed model attempts
to solve a variational total variation penalized optimization problem, and we
show that the optimal function can be linearly represented by a set of step
functions with the jump points known to as ahead of time. This allows us to
solve the problem by either Lasso or fused lasso efficiently using proximal
stochastic variance-reduced gradient algorithm. The results suggest that the
model significantly outperforms the classic Cox model and is highly interpretable.
Through a careful case study, we found that at least some of the active features
and jump points we discovered by fitting the model to data are indeed important
components of known vulnerability, and major jump points often clearly marks
out the life cycles of these exploits. Finally we show that our models are
much general hazard regression models and particularly suitable for responding
to external environment.

\section*{Acknowledgements}
The work was partially completed during ZL's visit to CMU from 2014 to 2016. The authors would
like to thank Ryan Tibshirani for useful discussions and for bringing our attention to the work
of \citep{parekh2015convex}, and Li Wang, Chaochao Chen in Ant Financial for providing their proprietary
data, and anonymous AISTATS reviewers for insightful comments that lead to substantial improvements of the paper.


\bibliographystyle{abbrvnat}
\bibliography{hazard}
\appendix
\section{Proofs of technical results.}
\begin{proof}[Proof of Theorem~\ref{thm:repre}]
	Let the feature $j$ of user $i$ at time $t$ be $x_{ij}(t)$ and nonnegative, the hazard function for user $i$
	$$\lambda(t) = \sum_{j=1}^d   x_{ij}(t) w_{j}(t)$$
	and the cumulative hazard function
	\begin{align}
	\Lambda(t) &= \int_{-\infty}^t\sum_{j=1}^d  x_{ij}(t) w_{j}(t) dt = \sum_{j=1}^d \int_{-\infty}^t x_{ij}(t) w_{j}(t) dt \nonumber\\
	&= \sum_{j=1}^d \left(\sum_{\tau\in \cT_j, \tau\leq t}\alpha_{j,\tau} W_j(\tau)   +  \alpha_{j,t} W_j(t)\right).\label{eq:paramerize_Lambda}
	\end{align}
	where $W_j(t) = \int_{-\infty}^t w_{j}(t) dt$, $\cT_j$ denotes all break points of the piecewise constant $x_{ij}(t)$ and $\alpha_{j,\tau}$ are coefficients that depends only on $x_{ij}$.
	When there are no uncensored observations, we can re-parameterize the above variational optimization problems \eqref{eq:variational} using the $\Lambda(t)$ hence $W_j(t)$ alone:
	\begin{equation}\label{eq:variational_cumulative}
	\begin{aligned}
	&\minimize_{(W_0,W_1,...,W_d) \in \cF^{d}} && \cL(\{\bm{\tau},\bm{\Psi},\bm{Z}\},\bm{W}) + \gamma \sum_{j=0}^d \TV(\frac{\partial}{\partial t}W_j)\\
	&\text{subject to } &&W_j(t)\geq 0, W_j(t+\delta) - W_j(t) \geq 0 \text{ for any } j\in [d],t\in \R, \delta\in \R_+.\\
	&&& \text{(if monotone) } W_j \text{ is convex, for any } j\in [d].  \\
	\end{aligned}
	\end{equation}
	Let $\cT$ be the set of observed time points (including $0,T$ and all censored interval boundaries).
	For each $j\in[d]$, let $W_j^*$ be the optimal solution to \eqref{eq:variational_cumulative}.
	
	By Proposition 7 of \citet{mammen1997locally}, we know that for each $j$, there is a spline $\tilde{W}_j$ of order $1$ such that 
	\begin{equation}\label{eq:spline_constraints}
	\left\{ \begin{aligned}
	& \text{All knots of the spline are contained in }\cT\backslash\{0,T\}\\
	&\tilde{W}_j(\tau) = W_j^*(\tau) \text{ for all } \tau\in \cT\\
	&\TV(\frac{\partial}{\partial t}\tilde{W}_j) \leq \TV(\frac{\partial}{\partial t}W^*_j)
	\end{aligned} \right.
	\end{equation}
	We will now show that $\tilde{W}_j$ also defines a set of optimal solution using these properties.
	
	Note that the loss function $\cL(\{\bm{\tau},\bm{\Psi},\bm{Z}\},\bm{W})$ can be decomposed into the sum of negative log-probability of form as described in 
	\eqref{eq:cumulative_repre}, and when there are no uncensored data, the value of the loss function is completely determined by the survival function $S(t)$ evaluated at $t\in \cT$. 
	There is a one-to-one mapping between survival functions and the cumulative hazard functions through
	$
	S(t) = \exp(-\Lambda(t)).
	$
	It follows from \eqref{eq:paramerize_Lambda} that $\cL(\{\bm{\tau},\bm{\Psi},\bm{Z}\},\bm{W})$ is a function of $\bm{W}$ only through its evaluations at $\bm{W}(\cT)$, therefore 
	$$\cL(\{\bm{\tau},\bm{\Psi},\bm{Z}\},\bm{\tilde{W}}) = \cL(\{\bm{\tau},\bm{\Psi},\bm{Z}\},\bm{W}^*).$$
	By $\TV(\frac{\partial}{\partial t}\tilde{W}_j) \leq \TV\frac{\partial}{\partial t}(W^*_j)$,	 we know that $\bm{\tilde{W}}$ has a smaller overall objective function than the optimal solution.
	
	It remains to show that $\bm{\tilde{W}}$ is feasible. First note that the only spline of order $1$ that satisfy the first and second condition is the piecewise linear interpolation of $W_j^{*}(\tau)$ the knots in $\cT$. 
	For each $j$, the constraints require that $W^*_j$ obeys that $W^*_j$ is non-negative, non-decreasing. This ensures that the piecewise linear interpolation of any subset of points in the domain of $W^*_j$ to be also nonnegative, monotonically nondecreasing, which ensures the feasibility of $\tilde{W}_j$.
	
	In the monotone model, the monotonic constraints on $w_j$ translates into a condition that says $W_j$ is convex. Since $W_j^*$ is feasible then it is convex. A piecewise linear interpolation of points on a convex function is also convex. This follows directly by checking that all the sublevel sets are convex sets.
	
	Finally, $\tilde{W}_j$ can be represented by a nonnegative linear combination of truncated power basis functions defined on $\cT$ and the corresponding hazard function $w_j$ can be represented by the same nonnegative combination of step functions defined at $\cT$. This completes the proof.
\end{proof}

\begin{proof}[Proof of Proposition~\ref{prop:prox_decomp}.]

	Define proximal operator $\text{Prox}_{f}(x)  :=  \argmin_{y} 0.5\|x-y\|^2 + f(y)$ for function $f$. 
	Theorem 1 of \citet{yu2013decomposing} states that 
	$
	\text{Prox}_{f+g}  =  \text{Prox}_f \circ \text{Prox}_g
	$
	if for any $x$, 
	\begin{equation}\label{eq:yu_condition}
	\partial g(\text{Prox}_f(x)) \supseteq \partial g(x).
	\end{equation}
	We will verify this condition for the ``Standard model'', namely, the case when $f(x) = \delta(x\geq 0)$ and $g(x) = \gamma\|Dx\|_1$ or $g(x) = \gamma\|Dx\|_1 + \delta(Dx \geq 0)$.
	
	In the first case 
	$$
	\partial g(x) =  \gamma D^T \left\{ u\in \R^{n-1} \middle| \begin{aligned} &u_i =  \sign(x_{i+1}-x_i)&& \text{ if } |[Dx]_i| >0\\
	& -1\leq u_i\leq 1 && \text{ otherwise;}\end{aligned}  \text{ for }i=1,...,n-1 \right\}
	$$
	It suffices to verify that for each $i$, if $[Dx]_i >0$ then $[D\text{Prox}_f(x)]_i\geq 0$ and if $[Dx]_i <0$ then $[D\text{Prox}_f(x)]_i\leq 0$.
	
	$[Dx]_i=  x_{i+1}-x_i$. $\text{Prox}_f(x)$ is simply the projection to the nonnegative cone. There are only for cases of $x_{i+1}$ and $x_i$. When at least one of them is positive, the projection does not change the sign of $x_{i+1}-x_i$, so the condition easily checks out. When both of them are not positive, $[D\text{Prox}_f(x)]_i = 0$ which only enlarges the subdifferential set. Therefore, \citeauthor{yu2013decomposing}'s condition \eqref{eq:yu_condition} is true for $f(x)=\delta(x\geq 0)$ and $g(x)= \gamma\|Dx\|_1$ and the proximal operator decomposes. 
	
	We now move on to check \eqref{eq:yu_condition} for the ``monotone model'', in this case $g(x) = \gamma\|Dx\|_1 + \delta(Dx\geq 0)$.
	\begin{align*}
	\partial g(x) &=  \partial_x\|Dx\|_1 + \partial_x \delta(Dx \geq 0) \\
	&= D^T \left\{ \gamma u+v  \middle| u\in \cS_u,  v\in  \cS_v \right\}
	\end{align*}
	where $\cS_1(x) = \left\{ u\in \R^{n-1} \middle| \begin{aligned} &u_i =  \sign(x_{i+1}-x_i)&& \text{ if } |[Dx]_i| >0\\
	& -1\leq u_i\leq 1 && \text{ otherwise;}\end{aligned}  \text{ for }i=1,...,n-1 \right\}$
	and the subdifferential of indicator function is
	$$\cS_2(x) = \begin{cases}	
	\left\{ v\in \R^{n-1} \middle| v^TDx \geq v^Ty \text{ for any }y\geq 0 \right\} & \text{ if } Dx\geq 0,\\
	\emptyset & \text{ otherwise. }\end{cases}
	$$
	When $x$ is feasible, the subdifferential is known as the ``normal cone''. 
	
	First of all, we have already shown that if $u\in \cS_1(x)$ then $u\in \cS_1(\text{Prox}_f(x))$.
	It remains to check that if $\cS_2(x)\subseteq \cS_2(\text{prox}_f(x))$.
	First of all if $x$ is not feasible, then the inclusion holds trivially. If $x$ is feasible, namely, $Dx\geq 0$, then we need to show that if $v\in \cS_2(x)$, then $v\in \cS(\text{prox}_f(x))$.
	
	The condition for $v\in \cS_2(x)$ also decomposes to every coordinate, since
	$$
	\sup_{y\geq 0} \langle v,y\rangle = \begin{cases}
	\infty, & \text{ if } \max_{i} v_i > 0\\
	0, &\text{ otherwise.}
	\end{cases}
	$$
	In other word, we have $v_i=0$ if $[Dx]_i > 0$, $v_i\leq 0$ if $[Dx]_i=0$ (check that this is an alternative definition of the normal cone geometrically).
	
	We now discuss the different cases of $x_{i+1}$ and $x_i$. Since $x$ is feasible,  there are only three cases of potential signs of $x_i$ and $x_{i+1}$. If $x_{i+1}$ and $x_i$ are both nonnegative, then $\text{prox}_{f}(x)=x$ and the constraint on $v_i $ remains unchanged. If $x_{i+1}$ and $x_i$ are both negative, then $[D\text{prox}_f(x)]_i = 0$ and the constraint on $v_i$ changes from possibly $v_i=0$ to $v_i\leq 0$. If $x_{i+1}\geq 0$ and $x_i<0$, then $[D\text{Prox}_f(x)]_i\geq 0$, so the constraints on $v_i$ is either unchanged  or changes from $v_i=0$ to $v_i\geq 0$.
	To conclude, in all cases, the subdifferential set is only enlarged when we replace $x$ with $\text{prox}_f(x)$. This checks \citeauthor{yu2013decomposing}'s sufficient condition \eqref{eq:yu_condition} for the monotone model and completes the proof. 
\end{proof}

\begin{proof}[Proof of Lemma~\ref{lem:tv-log2tv}]
	We prove by a direct calculation. 
	$$
	\xi(w) =  {\tilde{\TV}}_{\log}^\epsilon(w) - \frac{\|Dw\|_1}{\epsilon}
	$$
	always exists. It remains to prove differentiability. For any $h\in \R^{|\cT|}$, for convenience in notation, we do a change of variable and define $u:=Dw$ and $v:=Dh$.
	By Taylor's theorem, there exists $0 \leq \eta_i \leq |u_i+v_i|-|v_i|$ such that 
	$$\log(\epsilon + |u_i| + |u_i+v_i|-|u_i|)= \log(\epsilon + |u_i|) + \frac{|u_i+v_i|-|u_i|}{\epsilon + |u_i| } - \frac{(|u_i+v_i|-|u_i|)^2}{2(\epsilon + |u_i| + \eta_i)^2}.$$

	\begin{align*}
	\xi(w+h) &= \sum_{i=1}^{|\cT|-1} \left[\log(\epsilon +|u+v|_i)  - \frac{|u_i+v_i|}{\epsilon}\right]\\
	&=\sum_{i=1}^{|\cT|-1}\left[\log(\epsilon + |u_i|) + \frac{|u_i+v_i|-|u_i|}{\epsilon + |u_i| } - \frac{(|u_i+v_i|-|u_i|)^2}{2(\epsilon + |u_i| + \eta_i)^2} - \frac{|u_i+v_i|-|u_i|}{\epsilon} - \frac{|u_i|}{\epsilon}\right]\\
	&=\xi(w) + \sum_{i=1}^{|\cT|-1}\left[ \frac{|u_i+v_i|-|u_i|}{\epsilon + |u_i| }  - \frac{|u_i+v_i|-|u_i|}{\epsilon}  - \frac{(|u_i+v_i|-|u_i|)^2}{2(\epsilon + |u_i| + \eta_i)^2}\right]
	\end{align*}
	
	For any $w$, we decompose the coordinate of $Dw$ into $\cS=\{i\in[|\cT|-1] \;|\;  [Dw]_i=0\}$ and its complement $\cS^c$, and we look at the above summation. If $i\in \cS$, we get
	$$
	\left[\frac{|v_i|}{\epsilon} -\frac{|v_i|}{\epsilon} + \frac{(|v_i|)^2}{2(\epsilon+ \eta_i)^2}\right]  = \frac{(|v_i|)^2}{2(\epsilon+ \eta_i)^2} = O(|v_i|^2).
	$$

	Now we take limit. 
	\begin{align*}
	\lim_{h\rightarrow 0}\frac{\xi(w+h) - \xi(h)}{\|h\|} &= \lim_{h\rightarrow 0}\left[\sum_{i\in\cS^c}\frac{ \frac{|u_i+v_i|-|u_i|}{\epsilon + |u_i| }  - \frac{|u_i+v_i|-|u_i|}{\epsilon}}{\|h\|} +  \frac{O(|\cT|[Dh]_i|^2)}{\|h\|} \right]\\
	&= \lim_{h\rightarrow 0}\left[\sum_{i\in\cS^c}\frac{ \frac{|[D(w+h)]_i|-|[Dw]_i|}{\epsilon + |Dw_i| }  - \frac{|[D(w+h)]_i|-|[Dw]_i|}{\epsilon}}{\|h\|}\right]\\
	&=\sum_{i\in\cS^c} \left(\frac{1}{\epsilon+|Dw|}-\frac{1}{\epsilon}\right)\lim_{h\rightarrow 0}  \frac{|[D(w+h)]_i|-|[Dw]_i|}{\|h\|}
	\end{align*}

	Note that $[D w]_i\neq 0$, therefore $|[Dw]_i|= |\mathbf{e}_iDw|$ is  differentiable in $w$ and 
	$$
	\lim_{h\rightarrow 0}  \frac{|[D(w+h)]_i|-|[Dw]_i|}{\|h\|} = \langle \frac{h}{\|h\|}, \frac{\partial}{\partial w}|\mathbf{e}_iDw|\rangle = \langle \frac{h}{\|h\|}, D^T\mathbf{e}_i\sign([Dw]_i)\rangle
	$$
	
	In other word, we have
	$$
	\lim_{h\rightarrow 0} \frac{ \xi(w+h)-\xi(x) - \langle h, D\diag(\frac{1}{\epsilon+|Dw|}-\frac{1}{\epsilon})\sign(Dw)\rangle}{\|h\|} = 0
	$$
	which checks the definition of the differentiability for multivariate functions and the gradient is 
	$$D\diag(\frac{1}{\epsilon+|Dw|}-\frac{1}{\epsilon})\sign(Dw)$$
	as claimed.
	Since the gradient is a Lipschitz function in $w$, we conclude that $\xi$ is continuously differentiable.
\end{proof}

\end{document}